\documentclass[usenames,dvipsnames,final,12pt]{colt2018} % Anonymized submission
\usepackage{times}

\usepackage{graphicx}

\usepackage{floatrow}

\usepackage{upgreek}

\usepackage{xcolor}
\usepackage{calc}

\usepackage[normalem]{ulem} % for using strikethrough - i.e, for crossing out text

\newcommand{\vop}{VoP}
\newcommand{\bE}{\mathbb{E}}
\newcommand{\cF}{\mathcal{F}}
\newcommand{\dReal}{\mathbb{R}^d}
\newcommand{\Real}{\mathbb{R}}

\newcommand{\hth}{h_1}

\newcommand{\bt}{b_1}
\newcommand{\vt}{v_1}
\newcommand{\Xt}{X_1}
\newcommand{\Wt}{W_1}
\newcommand{\Tt}{\Upgamma_1}
\newcommand{\Mt}{M^{(1)}}
\newcommand{\st}{\alpha}
\newcommand{\mt}{m_1}
\newcommand{\lt}{q_1}
\newcommand{\et}{\epsilon_1}
\newcommand{\etg}{R_1^{\textrm{gap}}}
\newcommand{\Rto}{R_1^\textrm{out}}
\newcommand{\Rti}{R_1^{\textrm{in}}}
\newcommand{\thS}{\theta^*}
\newcommand{\rt}{\rho}
\newcommand{\rtS}{\rho^{*}}
\newcommand{\bart}{\bar{\theta}}
\newcommand{\Et}{E_1}
\newcommand{\Ei}{E_i}
\newcommand{\Kt}{K_1}
\newcommand{\tSol}[1]{\theta(#1, \tI{n_0}, \theta_{n_0})}
\newcommand{\zetD}{\zeta^{\dt}}
\newcommand{\zetM}{\zeta^{\md}}
\newcommand{\zetT}{\zeta^{\te}}
\newcommand{\EtD}{\Et^{\dt}}
\newcommand{\EtM}{\Et^{\md}}
\newcommand{\EtT}{\Et^{\te}}
\newcommand{\EiD}{\Ei^{\dt}}
\newcommand{\EiM}{\Ei^{\md}}
\newcommand{\LtD}{L_1^{\dt}}
\newcommand{\LtM}{L_1^{\md}}
\newcommand{\LtT}[1]{L_{1 #1}^{\te}}
\newcommand{\Lt}[1]{L^{\theta}_{#1}}

\newcommand{\vw}{v_2}
\newcommand{\Tw}{\Upgamma_2}
\newcommand{\Ww}{W_2}
\newcommand{\Mw}{M^{(2)}}
\newcommand{\sw}{\beta}
\newcommand{\mw}{m_2}
\newcommand{\lz}{q_2}
\newcommand{\lzp}{q_2^\prime}
\newcommand{\ez}{\epsilon_2}
\newcommand{\ezg}{R_2^{\textrm{gap}}}
\newcommand{\Rwo}{R_2^w}
\newcommand{\Rzo}{R_2^{\textrm{out}}}
\newcommand{\Rzi}{R_2^{\textrm{in}}}
\newcommand{\rz}{\nu}
\newcommand{\rzS}{\nu^{*}}
\newcommand{\chizD}{\chi^{\dt}}
\newcommand{\chizM}{\chi^{\md}}
\newcommand{\chizS}{\chi^{\sd}}
\newcommand{\barz}{\bar{z}}

\newcommand{\Ez}{E_2}
\newcommand{\Lz}{L^z}
\newcommand{\LzD}{L_2^{\dt}}
\newcommand{\LzM}{L_2^{\md}}
\newcommand{\LzS}{L_2^{\sd}}
\newcommand{\Kz}{K_2}
\newcommand{\EzD}{\Ez^{\dt}}
\newcommand{\EzM}{\Ez^{\md}}
\newcommand{\EzS}{\Ez^{\sd}}
\newcommand{\zSol}[1]{z(#1, \sI{n_0}, z_{n_0})}

\newcommand{\N}[1]{N_{0,#1}}
\newcommand{\n}[1]{N_{1,#1}}
\newcommand{\aftE}{\text{after}}
\newcommand{\midE}{\text{mid}}
\newcommand{\lmin}{q_{\min}}
\newcommand{\lm}{q}

\newcommand{\dt}{\text{de}}
\newcommand{\md}{\text{md}}
\newcommand{\sd}{\text{sd}}
\newcommand{\te}{\text{te}}
\newcommand{\df}{\mathrm{d}}

\newcommand{\Id}{\mathbb{I}}

\newcommand{\tI}[1]{t_{#1}}
\newcommand{\sI}[1]{s_{#1}}

\newcommand{\cE}{\mathcal{E}}
\newcommand{\cA}{\mathcal{A}}
\newcommand{\norm}[1]{\left\lVert#1\right\rVert}

\newcommand{\zS}{z^*}

\newcommand{\Rw}{R_2^w}
\newcommand{\Rs}{R^*}
\newcommand{\rl}{\textrm{real}}
\newcommand{\Jt}{J^\theta}
\newcommand{\Jz}{J^z}
\newcommand{\It}{I^\theta}
\newcommand{\Iz}{I^z}

\newcommand{\gal}[1]{#1}
\newcommand{\gugan}[1]{#1}

\title[Two-Timescale Stochastic Approximation]{Finite Sample Analysis of Two-Timescale Stochastic Approximation with Applications to Reinforcement Learning}

\usepackage{times}
\coltauthor{\Name{Gal Dalal}\thanks{Equal contribution.} \Email{gald@campus.technion.ac.il}
\addr Technion, Israel
\\
\Name{Bal\'azs Sz\"or\'enyi}\footnotemark[1] \Email{szorenyi.balazs@gmail.com}
	\addr  Yahoo Research, NYC
\\
\Name{Gugan Thoppe}\footnotemark[1]
	\Email{gugan.thoppe@gmail.com}
	\addr Duke University, USA
\\
\Name{Shie Mannor}
	\Email{shie@ee.technion.ac.il}
	\addr Technion, Israel
}

\begin{document}

%\author{Gal Dalal$^*$,
%Bal\'azs Sz\"or\'enyi$^*$,
%Gugan Thoppe$^*$,
%Shie Mannor% <-this % stops a space
%\thanks{M. Shell was with the Department
%of Electrical and Computer Engineering, Georgia Institute of Technology, Atlanta,

\maketitle


\begin{abstract}
Two-timescale Stochastic Approximation (SA) algorithms are widely used in Reinforcement Learning (RL). Their iterates have two parts that are updated using distinct stepsizes. In this work, we develop a novel recipe for their finite sample analysis. Using this, we provide a concentration bound, which is the first such result for a two-timescale SA. The type of bound we obtain is known as ``lock-in probability''. We also introduce a new projection scheme, in which the time between successive projections increases exponentially. This scheme allows one to elegantly transform a lock-in probability into a convergence rate result for projected two-timescale SA. From this latter result, we then extract key insights on stepsize selection.  As an application, we finally obtain convergence rates for the projected two-timescale RL algorithms GTD(0), GTD2, and TDC.
\end{abstract}

\section{Introduction}

Stochastic Approximation (SA) is the subject of a vast literature, both theoretical and applied \citep{kushner1997stochatic}. It is used for finding optimal points or zeros of a function for which only noisy access is available. Consequently, SA lies at the core of  machine learning; \gugan{in particular, it is widely used in Reinforcement Learning (RL) and, more so, when function approximation is used.}

\gal{A powerful, commonly used} analysis tool for SA algorithms \gal{is} the Ordinary Differential Equation (ODE) method \citep{borkar2000ode}. \gal{Its underlying idea} is that, under the right conditions, \gugan{the} noise effects \gugan{eventually} average out and the SA iterates \gugan{then} closely track the trajectory of the so-called ``limiting ODE''. \gal{The ODE method is classically used as} a convenient recipe for showing \gal{asymptotic SA} convergence. \gugan{The RL literature, therefore, has several results of such type,} especially when the state-space is large and function approximation is used \citep{sutton2009convergent,sutton2009fast,sutton2015emphatic,bhatnagar2009natural}. \gal{Contrarily, f}inite sample analyses for SA are scarce; in fact, they are nonexistent in the case of two-timescale SA. This \gal{provides} the motivation for our work.
%	\red{Do we? Not just for RL? - yes, but that's what is written}

\subsection{Related Work}\label{sec:related_rl_work}
%Two-timescale SA methods are prominent in RL \citep{peters2008natural,bhatnagar2009natural,sutton2009fast}. Nonetheless, as mentioned before, there are no finite sample analyses for these types of algorithms.

%\gugan{Below we briefly survey related results for single-timescale SA and asymptotic convergence results for two-timescale RL algorithms.}

A broad, rigorous study of SA is given in \citep{borkar2008stochastic}; in particular, it contains concentration bounds for single-timescale methods. A more recent work \citep{thoppe2015concentration} obtains tighter concentration bounds under weaker assumptions for single-timescale SA using a variational methodology called Alekseev's Formula. In the context of \gal{single-timescale RL}, \cite{kondathesis, korda2015td, dalal2018finite} discuss convergence rate\gal{s} for TD(0).

%Next, we relate to the relevant RL literature on two time-scale methods. We partition them into two principal classes:
\gal{Convergence rate results for two-timescale SA are, on the other hand, relatively scarce. Asymptotic convergence rates appear in \citep{spall1992multivariate,gerencser1997rate,konda2004convergence,mokkadem2006convergence}; these are of different nature than the finite-time analysis conducted in our work}. \gal{In the case of two-timescale RL methods, relevant literature} can be partitioned into two principal classes: actor-critic and gradient Temporal Difference (TD). In an actor-critic setting, a policy is being evaluated by the critic in the fast timescale, and improved by the actor in the slow time-scale; \gal{two asymptotic convergence guarantees appear in} \citep{peters2008natural,bhatnagar2009natural}. The second class, gradient TD methods, was introduced in \citep{sutton2009convergent}. This work presented the GTD(0) algorithm, which is \gal{a} gradient descent \gal{variant} of TD(0); being applicable to the so-called off-policy setting, it has a clear advantage over TD(0). Later variants, GTD2 \ and TDC, were reported to be faster than GTD(0) while enjoying its benefits. \gugan{These three methods} were shown to \gugan{asymptotically} converge in the case of linear and non-linear function approximation \citep{sutton2009convergent, sutton2009fast, bhatnagar2009convergent}.
\gugan{Separately,} there also exists a convergence rate result for altered versions of the GTD family \citep{liu2015finite}. There\gal{,} projections are used
 %for keeping the iterates in a convex set
and the learning rates are set to a fixed ratio. The latter makes the altered algorithms single-timescale variants of the original ones.

\subsection{Our Contributions}

Our main contributions are the following:
\begin{itemize}

\item \gugan{Inspired by \citep{borkar2008stochastic}, we develop a novel recipe for finite sample analysis of linear two-timescale SA. An initial key step here is a transformation of the iterates (see Remark~\ref{rem:AnalysisApproach}), which we believe can be elevated to general (non-linear) two-timescale settings. Then, by employing the Variation of Parameters method, we obtain a tighter bound on the distance between the SA trajectories and suitable limiting ODE solutions than the one handled in \citep{borkar2008stochastic}.  %do a Variation of Parameters, tight analysis, employ a martingale concentration bound.
    }

\item \gugan{Using the above recipe, we obtain a concentration bound for linear two-timescale SA (see Theorem~\ref{thm:condMain}); this is the first such result for two-timescale SA of any kind. In literature, such concentration bounds are also known as ``lock-in probability".}

\item \gugan{Additionally, we introduce a novel projection scheme, in which the time between successive projections progressively doubles\gal{; w}e refer to \gal{it} as ``sparse projection". This scheme enables one to elegantly transform a concentration bound, of the type we obtain, into a \gal{convergence} rate for projected two-timescale SA (see Theorem~\ref{thm:SparseProj}). We stress the strength of this tool in bridging the gap between two research communities\gal{:} those who are interested in lock-in probabilities/concentration bounds\gal{,} and those who care about convergence rates.}

\item \gugan{As an application, we obtain convergence rates for the sparsely projected variants of two-timescale RL algorithms: GTD(0), GTD2, and TDC. This is the first finite time result} for the above algorithms in their true two-timescale form (see Remark~\ref{rem: true 2TS}).

\item Finally, we do away with the usual square summability assumption on stepsizes (see Remark~\ref{remark: squre summability}). Therefore, our tool is relevant for a broader family of stepsizes. An example of its usefulness is Polyak-Ruppert-averaging with constant stepsizes \citep{defossez2014constant,lakshminarayanan2018linear}, whose behavior, we believe, is similar to two-timescale algorithms with slowly-decaying non-square-summable stepsizes (e.g., $n^{-\alpha}$ with $\alpha$ close to 0).

%	\item  Analysis of single-timescale follows automatically from our approach, and is applicable to all linear single-timescale SA.

\end{itemize}

%\gal{TODO: Do we want to add our $n^{-1/3}$ educated guess as a contribution? If not here, we should mention it in the intro and/or abstract}.

\section{Preliminaries}
\label{sec:2TSSetup}
\gugan{Here} we present the \gugan{linear} two-timescale SA \gugan{paradigm}, state our goal, and list our assumptions.

A generic \gal{linear} two-timescale SA is
\begin{eqnarray}
\theta_{n + 1} &  = & \theta_n + \st_n [h_1(\theta_n, w_n) + \Mt_{n + 1}] \enspace, \label{eq:theta_iter}\\
w_{n + 1} & = & w_n + \sw_n [h_2(\theta_n, w_n) + \Mw_{n + 1}] \enspace,  \label{eq:w_iter}
\end{eqnarray}
where $\st_n, \sw_n \in \Real$ are stepsizes, $M^{(i)}_n \in \dReal$ denotes noise, and \gugan{$h_i: \dReal \times \dReal \to \dReal$ has the form}
\begin{equation}
\label{eqn:lht}
h_i(\theta, w) = v_i - \Upgamma_i \theta - W_i w
\end{equation}
for a vector $v_i \in \dReal$ and matrices $\Upgamma_i, W_i \in \mathbb{R}^{d \times d}.$
\begin{remark} \label{rem: true 2TS}
In this work, we are interested in the analysis of a ``true two-timescale process''. By this, we mean that $\Upgamma_2$ ought to be invertible \gugan{and that $\alpha_n/\beta_n \to 0.$ The first condition couples the two iterates together; nevertheless, all the results in this work hold even without this restriction. The second condition is indeed assumed throughout (see \ref{assum:stepSize} below); we do not allow $\alpha_n/\beta_n$ to converge to a positive constant, as that would then turn \eqref{eq:theta_iter} and \eqref{eq:w_iter} into a single-timescale SA.}
\end{remark}

Our aim is to \gugan{finite time behaviour of} \eqref{eq:theta_iter} and \eqref{eq:w_iter} under the following assumptions.

{\renewcommand*\theenumi{$\pmb{\cA_\arabic{enumi}}$}
\begin{enumerate}
\item \label{assum:posDef} $\Ww$ and $\Xt := \Tt - \Wt \Ww^{-1} \Tw$ are positive definite (not necessarily symmetric).

\item \label{assum:stepSize} Stepsize sequences $\{\st_n\},~\{\sw_n\},$ and $\{\eta_n := \st_n/ \sw_n\}$ satisfy
\gugan{
\begin{equation}
\label{eqn:stepSizeBound}
\sum_{n = 0}^{\infty} \st_n = \sum_{n = 0}^{\infty} \sw_n =  \infty, \; \; \st_n,\sw_n,\eta_n   \leq 1,\mbox{ and } \lim_{n \to \infty} \st_n = \lim_{n \to \infty} \sw_n = \lim_{n \to \infty} \eta_n = 0.
\end{equation}
}
%
%		\begin{equation}
%		\label{eqn:2TSCond}
%		
%		\end{equation}
%
%where $\eta_n := \st_n/ \sw_n.$

\item \label{assum:Noise} $\{\Mt_n\}, \{\Mw_n\}$   are martingale difference sequences w.r.t. the \gugan{family} of $\sigma-$fields $\{\cF_n\},$ where
$
\cF_n = \sigma(\theta_0, w_0, \Mt_1, \Mw_1, \ldots, \Mt_n, \Mw_n).
$
\gugan{There exist} constants $\mt, \mw > 0$ so that
$
\norm{\Mt_{n + 1}} \leq \mt(1 + \norm{\theta_n} + \norm{w_n})$ and $ \norm{\Mw_{n + 1}} \leq \mw(1 + \norm{\theta_n} + \norm{w_n})
$
for all $n \geq 0.$
\end{enumerate}
}

\begin{remark}\label{remark: squre summability}
Notice that, unlike most works, $\sum_{n \geq 0} \st_n^2$ or $\sum_{n \geq 0} \sw_n^2$ need not be finite. Thus\gal{,} our analysis is applicable for a wider class of stepsizes; e.g., $1/n^\kappa$ with $\kappa \in (0,  1/2].$ \gugan{In \citep{borkar2008stochastic}, on which much of the existing RL literature is based on, the square summability assumption is due to the Gronwall inequality based approach. In contrast, for the specific setting here, we do a tighter analysis using the  variation of parameters formula \citep{lakshmikantham1998method}.}
\end{remark}

\gugan{We now briefly outline the ODE method from \cite[pp. 64-65]{borkar2008stochastic} for the analysis of \eqref{eq:theta_iter} and \eqref{eq:w_iter}, and also describe how our approach builds upon it. Since $\eta_n \to 0,$ $\{w_n\}$ is the fast transient and $\{\theta_n\}$ is the slow component. Therefore, the ODE that \eqref{eq:w_iter} might be expected to track is
\begin{equation}
\label{eqn:wLimODE}
\dot{w}(t) = \vw - \Tw \theta - \Ww w(t)
\end{equation}
for some fixed $\theta,$  and the ODE that \eqref{eq:theta_iter} might be expected to track is
\begin{equation}
\label{eqn:tLimODE}
\dot{\theta}(t) = \hth(\theta(t), \lambda(\theta(t))) = \bt - \Xt \theta(t),
\end{equation}
where $\bt := \vt - \Wt \Ww^{-1} \vw$ and $\lambda(\theta) := \Ww^{-1}[\vw - \Tw\theta].$
Due to \ref{assum:posDef}, the function $\lambda(\cdot)$ and $\bt$ are well defined.} Moreover, $\lambda(\theta)$ and $\thS := \Xt^{-1}\bt$ are unique globally asymptotically stable equilibrium points of \eqref{eqn:wLimODE} and \eqref{eqn:tLimODE}, respectively.

Lemma 1, \cite[p. 66]{borkar2008stochastic}, applied to \eqref{eq:theta_iter} and \eqref{eq:w_iter} gives $\lim_{n \to \infty} \norm{w_n - \lambda(\theta_n)} = 0$ under suitable assumptions. \gugan{Inspired by this,} we work with $\{z_n\}$ here instead of $\{w_n\}$ directly, where
{
\begin{equation}
\label{eq: z_n def}
z_n := w_n - \lambda(\theta_n) \enspace.
\end{equation}
}
Due to \eqref{eq:w_iter}, $\{z_n\}$ satisfies the \gugan{update} rule
\begin{equation}
\label{eq:z_iter}
z_{n + 1} = z_n - \sw_n \Ww z_n + \sw_n \Mw_{n + 1} + \lambda(\theta_n) - \lambda(\theta_{n + 1}) \enspace.
\end{equation}
\gugan{Hence, and as $\{\theta_n\}$ is the slow component, the limiting ODE that \eqref{eq:z_iter} might be expected to track is}
\begin{equation}
\label{eqn:zLimODE}
\dot{z}(s) = -\Ww z(s) \enspace.
\end{equation}
\gugan{As} $\Ww$ is positive definite (see \ref{assum:posDef}), $\zS = 0$ is the  globally asymptotically stable equilibrium of \eqref{eqn:zLimODE}.

\begin{remark}
\label{rem:AnalysisApproach}
\gugan{Using} $\{z_n\}$ instead of $\{w_n\}$ is the main reason why our approach works. \gugan{Observe that the limiting ODE in \eqref{eqn:wLimODE} varies as $\theta_n$ evolves;} \gugan{in contrast}, \eqref{eqn:zLimODE} remains unchanged. \gugan{Hence,} comparing \eqref{eq:z_iter} with \eqref{eqn:zLimODE} \gugan{is} easier than comparing \eqref{eq:w_iter} with \eqref{eqn:wLimODE}. \gugan{While this idea is indeed inspired by \cite[Lemma 1, p. 66]{borkar2008stochastic}, there \eqref{eq:z_iter} and \eqref{eqn:zLimODE} are not required to be explicitly dealt with.}
\end{remark}

\section{Main Results}

\begin{table}[t]
{\small
\begin{tabular}{l|l||l|l}
\hline
\gal{Constant}& \gal{Source}
&
\gal{Constant}& \gal{Source}
\\
\hline
$\Kt$, $\Kz$ & \eqref{eqn:zMatrixBd}, \eqref{eqn:tMatrixBd}
&
$\LtT{b} = \Kt \norm{\Wt}\norm{\Ww} \Rzi/ \lt$& Lemma~\ref{lem:EtDBd}
\\
$\lt, \lz$ & above \eqref{eq: time definition}
&
$\LtT{c} = \Kt \norm{\Wt} /\lt$& Lemma~\ref{lem:EtDBd}
\\
$\lmin=\min\{\lt,\lz\}$& \eqref{eq:lmdef}
&
$\LtM = \Kt \mt[1 + \Rs + \Rto + \Rw]$& Lemma~\ref{lem:EtDBd}
\\
$\lm \in (0, \lmin)$ & \eqref{eq:lmdef}
&
$\LtD = \frac{\Kt \norm{\Xt} \Jt}{\lt}$& Lemma~\ref{lem:EtDBd}
\\
$\Rto = \Rti + \frac{4\Kt \norm{\Wt} \Kz \Rzi}{(\lmin-\lm)e}$ & {\eqref{eq: defn of Rto}} %\gal{Theorem~\ref{thm:condMain}}
&
$\Lt{a} = \LtT{a}$& Lemma~\ref{lem:rtBd}
\\
$\Rs = \norm{\Xt^{-1}} \norm{\bt}$ & {\eqref{eq:R-star defn}}
&
$\Lt{c} = \LtT{c}$& Lemma~\ref{lem:rtBd}
\\
$\Rwo =\Rzo + \norm{\Ww^{-1}}$ & \eqref{eqn:wBd}
&
$\Lt{b} = \LtD + \LtM +  \norm{\Xt} \Rti + \LtT{b}$& Lemma~\ref{lem:rtBd}
\\
\hfill$\times \big[\norm{\vw} + \norm{\Tw}[\Rs + \Rto] \big]$&
&
$\LzM = \Kz\mw[1 + \Rs + \Rto + \Rw]$& Lemma~\ref{lem:EzDBd}
\\
$\etg = \Rto-\Rti$, {$\ezg = \Rzo-\Rzi$} & {\eqref{eq:epsilon smaller than R}}
&
$\LzS = \Kz\frac{\norm{\Ww^{-1}}\norm{\Tw}\Jt}{\lz}$& Lemma~\ref{lem:EzDBd}
\\
$\Jt = \norm{\Tt}[ \Rs  + \Rto] + \norm{\Wt} \Rw$
& Lemma~\ref{lem:ItkBd}
&
{$\LzD = \Kz\frac{\norm{\Ww}\Jz}{ \lz}$}& {Lemma~\ref{lem:EzDBd}}
\\
~~\hfill$+ \norm{\vt}  + \mt [1 + \Rs + \Rto + \Rw]$ &
&
$\Lz= \norm{\Ww} \Rzi + \LzD + \LzS + \LzM$& Lemma~\ref{lem:rzBd}
\\
$\Jz = \norm{\Ww} \Rzo  +  \norm{\Ww^{-1}} \norm{\Tw}\Jt$
& Lemma~\ref{lem:ItkBd}
&
\gal{$c_1 = (16 \Kt^2 d^3 [\LtM]^2 )^{-1}$} & Theorem~\ref{thm:condMain}
\\
\hfill$+ \mw(1 + \Rs + \Rto + \Rw)$ &
&
\gal{$c_2 = (9 \Kz^2 d^3 [\LzM]^2)^{-1}$} & Theorem~\ref{thm:condMain}
\\
$\LtT{a} = \Kt \norm{\Wt} \Kz \Rzi  \frac{1}{(\lmin-\lm)e}$& Lemma~\ref{lem:EtDBd}
&
\gal{$c_3 = (64 \Kz^2 [\Lt{c}]^2 d^3 [\LzM]^2)^{-1}$} & Theorem~\ref{thm:condMain}\\[0.5ex]
\hline
\end{tabular}
}
\caption{\label{tab: constants} A summary of constants and where they are defined. \gugan{Here $\mt, \mw >0$ are as in \ref{assum:Noise}, and $\Rti>0$ and $\Rzo>\Rzi>0$ are constants chosen as} { in Theorems~\ref{thm:condMain} and \ref{thm:SparseProj}}.
Note that constants in the left column do not depend on constants in the right column. Similarly, no constant depends on constants below it in the same column, or on $\et,$ $\ez$, {  $\{\st_k\}$ or $\{\sw_k\}$}.}
\end{table}

\gal{In this section, we give our two main results on two-timescale stochastic approximation} and \gugan{also introduce our projection scheme}. The first \gugan{result} is a general concentration bound for any stepsizes satisfying \ref{assum:stepSize}. \gugan{This result concerns the behavior of a two-timescale SA from some time index $n_0$ onwards and requires that the iterates be bounded at $n_0$. This is in the spirit of most existing concentration bounds/lock-in probability results for single-timescale methods \citep{borkar2008stochastic, thoppe2015concentration}. By projecting the iterates of a two-timescale SA via our novel projection scheme, we then transform our above concentration bound into a convergence rate result. This latter result applies for all time indices and the boundedness assumption holds here due to projections.}

\gal{
\subsection{A General Concentration Bound}
 }

Let $\lt,\lz>0$ be lower bounds on the real part of the eigenvalues of matrices $\Xt$ and $\Ww$, respectively. \gugan{For} $n \geq 0,$ let $a_n := \sum_{k = 0}^{n - 1} \st_k^{2} e^{ -2 \lt \sum_{i = k+1}^{n - 1}\alpha_i}$ and
$b_n := \sum_{k = 0}^{n - 1} \sw_k^{2} e^{ -2 \lz \sum_{i = k+1}^{n - 1}\beta_i}$. These sums are obtained from the Azuma-Hoeffding concentration bound \gugan{that we use later}. Also, let \begin{equation}
\label{eq: time definition}
s_n := \sum_{k=0}^{n-1} \beta_k, \quad \mbox{and} \quad t_n := \sum_{k = 0}^{n - 1} \alpha_k \enspace.
\end{equation}
\gugan{Theorem~\ref{thm:condMain} gives our concentration bound; the additional terms in it are defined in} Tables~\ref{tab: constants} and \ref{tab: epsilon dependent constants}.

%\subsection{Constants}
%\label{sec: constants}
%Dependencies on system parameters and constants $\Rti,\Rto,\Rzi,\Rzo,\lm$ and $\lmin$ are not stressed in the notations as these parameters are either fix or can be chosen to be fixed.
%We do highlight the dependencies on $\et$ and $\ez$ and the stepsizes though.

%\subsubsection{Constants Independent of $\et, \ez$ and the Stepsizes}

%\subsubsection{Terms Depending on $\et, \ez$ and the Stepsizes}
\label{sec: dependencies on stepsizes}

%\gugan{Suppose $\et$ and $\ez$ are chosen arbitrarily from $\left(0,\;\min\left\{\Rti/4, 4 \Lt{a}\right\}\right)$ and $(0,\Rzi/4)$, respectively.}

\begin{theorem}[Main Technical Result] \label{thm:condMain}
	\gugan{Fix some constants $\Rti,\Rzi > 0$ {and $\Rzo>\Rzi$}.
		Pick $\et \in \left(0,\min\left\{\Rti, 4 \Lt{a}\right\}\right)$ and $\ez \in (0,{\min(\Rzi,\Rzo-\Rzi)})$}.
	 \gal{Fix some $n_0 \geq N_0$ and $n_1 \geq N_1$, where $N_0 \equiv N_0(\et, \ez, \{\alpha_k\}, \{\beta_k\})$  and $N_1 \equiv N_1(n_0,\et,\ez,\{\alpha_k\},\{\beta_k\})$ are as in Table~\ref{tab: epsilon dependent constants}. Consider the process defined by \eqref{eq:theta_iter} and \eqref{eq:w_iter} for $n \geq n_0,$ initialized at  arbitrary $\theta_{n_0},w_{n_0} \in \mathbb{R}^d$ such that
		\begin{equation}
		\label{eq: Gn' definition}
		\norm{\theta_{n_0} - \thS} \leq \Rti \mbox{ and } \norm{z_{n_0}} \leq \Rzi,
		\end{equation}
		where $z_{n_0}$ is as in \eqref{eq: z_n def}.
	}
	Then,
	\begin{multline}
\label{eqn:ProbBound}
	\Pr
	\{
	\norm{\theta_n - \thS} \leq \et, \norm{z_n} \leq \ez, \forall n \geq n_1\}
	\\
	\geq
	1 -2d^2 \! \!
	\sum_{n \geq n_0}\left[\exp\!\left[\tfrac{-c_1 \et^2}{a_n}\right] \!
	+ \exp\!\left[\tfrac{- c_2 \et^2}{b_n}\right]\!
	+ \exp\!\left[ \tfrac{- c_3\ez^2}{b_n}\right] \right]\!.
	\end{multline}%
	where
	{
		$c_1 = (16 \Kt^2 d^3 [\LtM]^2 )^{-1},$
		$c_2 = (9 \Kz^2 d^3 [\LzM]^2)^{-1}$, and
		$c_3 = (64 \Kz^2 [\Lt{c}]^2 d^3 [\LzM]^2)^{-1}$
		%$c_1, c_2, c_3 > 0$ are suitable constants (independent of $\et,\ez,\{\alpha_k\}$ and $\{\beta_k\}$).
		are constants independent of $\et,\ez,\{\alpha_k\}$ and $\{\beta_k\}$.}
	%is the event that the iterates begin bounded in $\Rti$ and $\Rzi$ at time $n_0$ (see \eqref{eq: G_n' definition}).
\end{theorem}

\gugan{\begin{proof}
See Section~\ref{sec: analysis outline} for the outline of the proof, and Appendix~\ref{sec:SupMat} for \gal{the detailed proof}.
\end{proof}}

\begin{table}[t!]	
{\small
\begin{tabular}{l | l}
\hline
Term  & Definition
\\[0.5ex]
\hline
$\N{a} \equiv \N{a}(\et,\ez,\{\alpha_k\},\{\beta_k\})$   & $ \min\left\{ N : \max \left\{\sup_{k \geq N} \beta_k, \; \sup_{k \geq N} \eta_k \right\} \leq \dfrac{\min\left\{\et/ 8, \ez/3\right\}}{\Lz \max\{\Lt{c}, 1\}}\right\}$ \\[2ex]
$\N{b} \equiv \N{b}(\et,\{\beta_k\})$ & $\min \left\{ N: \sup_{k \geq N} \beta_k \leq \et/(4\Lt{b}) \right\}$\\[0.5ex]
$N_0 \equiv N_0(\et, \ez, \{\alpha_k\}, \{\beta_k\})$ & $\max\{\N{a}, \N{b}\} $ \\[0.5ex]
$\n{a} \equiv \n{a}(n_0,\et, \{\alpha_k\})$ & $ \min\{j \geq n_0:[\Kt \Rti + \Lt{a}] e^{-\lm(\tI{j} - \tI{n_0})} \leq \et/4\}$ \\[0.5ex]
$\n{b} \equiv \n{b}(n_0,\ez, \{\beta_k\})$ & $\min\{j \geq n_0 : \Kz \Rzi e^{-\lz (\sI{j} - \sI{n_0})} \leq \ez/3\}$ \\[0.5ex]
$N_1 \equiv N_1(n_0,\et,\ez,\{\alpha_k\},\{\beta_k\})$ & $\max\{\n{a}, \n{b}\}$\\[0.5ex]
\hline
\end{tabular}
}
\caption{\label{tab: epsilon dependent constants} A summary of terms depending on $\et, \ez$, and the stepsize sequences, as appearing in the main theorems.
{These terms are formally introduced in Lemmas~\ref{lem:rzEqEv and rtEqEv} and \ref{lem:rzSEqEv and rtSEqEv}.
}	
% in Appendix~\ref{sec: conc_bound_appendx}.
}
\end{table}

{
Section~\ref{sec:2TSSetup} already discusses the close relation between the SA iterates \gugan{$\{\theta_n\}$ and $\{z_n\}$} and the corresponding ODE trajectories, which suggests that the analysis of the former should be based on the latter.
However, the sole fact that the ODE trajectories approach their respective solutions does not guarantee the same for the SA trajectories. The latter may drift away due to several factors (e.g., martingale noise), as discussed in Subsection~\ref{subsec:Comparison}.
However, Theorem~\ref{thm:condMain} makes it clear that,
%However, as depicted in Theorem~\ref{thm:condMain},
w.h.p., this does not happen.
These subtleties are discussed in more details in the following remark.
}

\begin{remark} \label{remark: main thm remark}
%The result
{Theorem~\ref{thm:condMain}}
\gugan{involves} two key notions \gugan{introduced} in Table~\ref{tab: epsilon dependent constants}: $N_0$ and $N_1$.
\begin{enumerate}
\item A large $N_0$ ensures the stepsizes are small enough to mitigate the factors that may cause the SA trajectories to drift. In the case of martingale \gugan{difference} noise, this can be directly seen from the terms $\st_n \Mt_{n+1}$ and $\sw_n \Mw_{n+1}$ in \eqref{eq:theta_iter} and \eqref{eq:z_iter}. %This concept is of the nature of previous single timescale finite sample analyses: Corollary 14, Chapter 4, \citep{borkar2008stochastic}, and the more recent \citep{korda2015td}, Theorem~1.

%A large $N_0$ ensures the stepsizes are small enough to mitigate the  martingale noise of the SA trajectories; i.e., the additive SA noise $\{\st_n \Mt_{n+1}\}$ and $\{\sw_n \Mw_{n+1}\}$ $ \forall n \geq N_0$ is small w.h.p.
\item The term $N_1$ is an intrinsic property of the limiting ODEs. It quantifies the number of iterations required by the two ODE trajectories to hit the $\epsilon$-neighborhoods of their respective solutions (and \gugan{stay} there) when started \gugan{in $\Rti$ and $\Rzi$ radii balls}. As shown in \gal{Theorem~\ref{thm:condMain}}, $N_1$ depends on \gugan{$n_0.$} A larger \gugan{$n_0$} means smaller stepsizes, which implies that a longer time is required for the trajectories to hit the $\epsilon$-neighbourhoods, in turn making $N_1$ larger.

    %As shown in \gal{Theorem~\ref{thm:condMain}}, $N_1$ also depends on $N_0$, since larger $N_0$ implies smaller stepsizes and hence more iterations.
\end{enumerate}
% The sole fact that the ODE trajectories reach their respective solutions does not guarantee the same for the SA trajectories. The latter may drift away due to several factors (e.g., martingale noise), as discussed in Subsection~\ref{subsec:zBd}. However, as depicted in Theorem~\ref{thm:condMain}, w.h.p. this does not happen.
\end{remark}

\subsection{\gugan{A Bound for Sub-exponential Series}}

In order to make Theorem~\ref{thm:condMain} more applicable, we derive closed form expressions for the r.h.s. of \eqref{eqn:ProbBound} for the case of inverse polynomial stepsizes; see Appendix~\ref{sec: bound for subexponential series}. In particular, we obtain a bound on the generic expression $\sum_{n = n_0}^{\infty} \exp[- B n^{p}],$ where $B \geq 0$ and $p \in (0,1).$  Such expressions are common in SA analyses. Thus, this result can be useful on its own.

\subsection{Convergence Rate of Sparsely Projected Iterates}
\gugan{Here we first describe our projection scheme, following which we give our convergence rate result in Theorem~\ref{thm:SparseProj}. In this latter result, we work with a specific family of stepsizes to obtain concrete closed-form expressions for the rate of convergence.}

For $n$ that is a power of 2, let $\Pi_{n,R}$ denote the projection into the $R$-ball; for every other $n,$ let $\Pi_{n,R}$ denote the identity{, where $R>0$ is some arbitrary constant}.  {We call this sparse projection as we project only on indices which are powers of $2$.}
\gal{
With $\theta_0', w_0' \in \dReal,$ let
\begin{eqnarray}
\theta_{n + 1}' & = & \Pi_{n + 1,\Rti/2}\Big(\theta_n' + \st_n [h_1(\theta_n', w_n') + M_{n + 1}^{(1')}]\Big) \enspace, \label{eq:thetap_iter} \\
w_{n + 1}' & = & \Pi_{n + 1,\Rzi/2}\Big(w_n' + \sw_n [h_2(\theta_n', w'_n) + M_{n + 1}^{(2')}]\Big) \label{eq:wp_iter}
\end{eqnarray}
denote the sparsely projected variant of \eqref{eq:theta_iter} and \eqref{eq:w_iter}%
}
{, where $\{M^{(1')}_{n}\}$ and $\{M^{(2')}_{n}\}$ are martingale difference sequences satisfying assumption \ref{assum:Noise}, just like $\{M^{(1)}_{n}\}$ and $\{M^{(2)}_{n}\}$.} The \gugan{idea} of projection is indeed common \citep{borkar2008stochastic,kushner1980projected}; \gugan{but, the novelty here is in doing only exponentially infrequent projections.} As seen in the proof of Theorem~\ref{thm:SparseProj}, this significantly simplifies our analysis.

{ \gugan{We now} introduce a carefully chosen instantiation of $N_0$ (where $N_0$ is \gugan{as} in Table~\ref{tab: epsilon dependent constants}) for the stepsize choice in Theorem~\ref{thm:SparseProj} below. This choice also regulates the $N_1$ term in an appropriate way, as we show later in the theorem's analysis.}
%\gal{Additionally, we introduce the following closed-form estimate for $N_0$ (presented in Fig.\ref{tab: epsilon dependent constants}), given the stepsize choice in Theorem~\ref{thm:SparseProj} below.}
For some $1>\alpha>\beta>0$ and $\epsilon \in (0,1)$, \gugan{let
	
	\begin{multline}
	N_0'(\epsilon,\alpha,\beta)
	:=
	\max \Bigg\{\left[ \tfrac{8\Lz}{\epsilon} \max\left\{\Lt{c}, 1\right\}
	\right]^{\frac{1}{\min\{\beta,\alpha-\beta\}}},  \left[\tfrac{4\Lt{b}}{\epsilon}\right]^{1/\beta},  \\
	\left[
	\tfrac{1-\alpha}{((1.5)^{1-\alpha}-1)\lm}
	\ln\tfrac{4[\Kt \Rti + \Lt{a}]}{\epsilon}
	\right]^{\frac{1}{1-\alpha}}, \left[
	\tfrac{1-\beta}{((1.5)^{1-\beta}-1)\lz}
	\ln\tfrac{3\Kz \Rzi}{\epsilon}
	\right]^{\frac{1}{1-\beta}}, 3
	\bigg\} \enspace.
	\label{eq: N0' defn}
	\end{multline}
}

\begin{theorem}[Finite Time Behavior of Sparsely Projected Iterates]
\label{thm:SparseProj}
\gugan{
Fix $\Rti, \Rzi >0.$ Suppose
\begin{equation}
\label{eq: theta assumption}
\|\theta^*\| \leq \Rti/4 \enspace, \mbox{ and }
\end{equation}
\begin{equation}
\label{eq: w assumption}
\{\theta \in \dReal: \|\theta\| \leq \Rti/2\}
\subseteq \{\theta \in \dReal : \|\lambda(\theta)\|\leq \Rzi/4\} \enspace.
\end{equation}
Let $\alpha_n=(n+1)^{-\alpha}$ and $\beta_n=(n+1)^{-\beta}$ with $1>\alpha>\beta>0.$
}
\gugan{Then the following hold.}
\gugan{
\begin{enumerate}
\item \label{st:ProbEst}For any {$\Rzo>\Rzi$,} $\epsilon \in (0,\min\{\Rti/4, \Rzi/4, 4 \Lt{a} {, \Rzo-\Rzi} \}),$ and $n_0' \geq N_0'(\epsilon,\st,\sw),$ such that $n_0'$ is a power of 2 and
$N_0'(\epsilon,\alpha,\beta) = O\left(\epsilon^{-\frac{1}{\min\{\beta,\alpha-\beta\}}}\right)$
is as in \eqref{eq: N0' defn}, we have
\begin{multline}
\Pr\{\norm{\theta_n' - \thS} \leq \epsilon, \norm{z_n'} \leq \epsilon, \forall n \geq 2n_0'\} \\
\geq
1
-
 \frac{2 d^2 c_{7a}}{\epsilon^{2/\alpha}} \; \exp\left[c_{5a} \epsilon^2 - c_{6a} \epsilon^2 (n_0')^\alpha \right] - \frac{4d^2 c_{7b}}{\epsilon^{2/\beta}} \; \exp\left[c_{5b} \epsilon^2 - c_{6b} \; \epsilon^2 (n_0')^\beta \right] \enspace,
\label{eq: bound on projected iterates}
\end{multline}
where $c_4 := \min\{c_2, c_3\}, c_{5a} = c_5(c_1, \kappa, \alpha, q_1), c_{5b} = c_5(c_4, \kappa, \beta, q_2),$ and so on for $c_{6a}, c_{7a},$ $c_{6b},$ and $c_{7b}.$ \gugan{The terms  $c_5, c_6,$ and $c_7$ are as defined in Lemma~\ref{lem: B formula}.}~\footnote{{Consult Table~\ref{tab: constants} for the rest of the constants, such as $c_1$, $c_2$ and $c_3$.}}
%Here $c_1, c_2, c_3$ are as defined in Theorem~\ref{thm:condMain} and $c_5, c_6,$ and $c_7$ are as defined in Lemma~\ref{lem: B formula}.
%
\item \label{st:ConvRate} There is some constant $C>0$ such that, for all $n > 3$ and $\delta \in (0,1),$ it holds that \footnote{An explicit expression for $C$ can be derived from the proof of Theorem~\ref{thm:SparseProj} which, for brevity, we haven't introduced. }
\begin{equation}
\label{eqn:FiniteTimeRate}
\Pr\left\{\max\{\|\theta'_n - \theta^*\|, \|z'_n\|\} \leq C\max\left[{n^{-{\beta}/{2}}} {\sqrt{\ln ({n}/{\delta)}}}\;,\; {n^{-(\st-\sw)}}\right]\right\}
\geq 1 - \delta \enspace.
\end{equation}
\end{enumerate}
}
\end{theorem}

\gugan{
\begin{proof}
See Appendix~\ref{sec:Proof_Sparse_Proj}.
\end{proof}
}

\begin{remark}
	To the best of our knowledge, the only other work that provides
	a high probability convergence rate for
	%a finite sample analysis to a
	projected SA algorithm is \citep{liu2015finite}, and it also assumes \eqref{eq: theta assumption}. Without this assumption, one would be required \gugan{to} study the convergence to the closest point to $\thS$ within $\Rti/4.$
	%Moreover, Assumption \eqref{eq: w assumption} is not less reasonable.
	%{ Is this really true? We should be more convincing - and check whether this is true at all, or if it is in fact less reasonable, but still acceptable.}
	%$\Rzi$ can be explicitly calculated to satisfy \eqref{eq: w assumption} using $\lambda(\theta) := \Ww^{-1}[\vw - \Tw\theta].$
	Assumption \eqref{eq: w assumption} can be easily seen to hold if $\Rzi$ is set to $4\|\Ww^{-1}\vw\|+2\Rti\|\Tw\|$ or greater.
\end{remark}

\gal{
\begin{remark}
	In continuation of Remark~\ref{remark: squre summability}, notice $\alpha$ and $\beta$ are not constrained to lie in $(1/2,1)$. This is, to the best of our knowledge, in contrast with any other two-timescale analysis in the literature.
\end{remark}
}
\gal{
\begin{remark}
	\label{rem: optimal rate}
Clearly, the tightest possible upper bound in \eqref{eqn:FiniteTimeRate} approaches $O(n^{-1/3})$ as $\alpha$ and $\beta$ simultaneously approach $1$ and $2/3$, respectively. 
We now briefly discuss the origin of the two limiting terms \gugan{there}. The $n^{-\beta/2}$ term stems from the convergence of \eqref{eq:wp_iter}, and corresponds to the known $n^{-1/2}$ rate limit of any single-timescale SA. It is also in line with Theorem~3.1 in \citep{dalal2018finite} for generic $\beta$. We note that a similar $n^{-\alpha/2}$ rate, stemming from the convergence of \eqref{eq:thetap_iter},  originally appears in the proof of Theorem~\ref{thm:SparseProj}; however, we drop it from the statement since $\alpha > \beta$. Separately, \gugan{the} $n^{-(\alpha - \beta)}$ term stems from the interaction between the $\theta$ and $z$ iterates, originating in \gugan{the last two terms in \eqref{eq:z_iter}}.
%our true two-timescale process (see Remark~\ref{rem: true 2TS}).
As \gugan{discussed above} Remark~\ref{rem:AnalysisApproach}, the slow component $\{\theta_n\}$ evolves in the $\st_n$ timescale; yet, it is part of $z_n$ update rule, which evolves in the $\sw_n$ timescale. Hence, the slow drift error (\gugan{see} Subsection~\ref{subsec:Comparison}) is governed by the \gugan{stepsize ratio} $\alpha_n /\beta_n$  (yielding the $\alpha - \beta$ here).
\end{remark}

}
%\begin{remark}
%{The upper bound in Theorem~\ref{thm:SparseProj} gets closer and closer to $n^{1/3}$ as $\alpha$ and $\beta$ simultaneously approach 1 and 2/3, respectively. The rate $n^{1/3}$, however, cannot be beaten. This is in contrast with the single time-scale setting, where the optimal rate is known to be $n^{1/2}$ under various settings. This raises the question, whether the rate $n^{1/3}$ is an inherent theoretical limit for the two time-scale stochastic processes, or if this limitation is simply an artifact of our approach.}
%\end{remark}

%\begin{remark}
	\gal{In \gugan{transforming} Theorem~\ref{thm:condMain} to Theorem~\ref{thm:SparseProj}, $N_0'$ (\gugan{see} \eqref{eq: N0' defn}) inherits \gugan{the} properties of both $N_0$ and $N_1$ from Theorem~\ref{thm:condMain}, whose roles we have portrayed in  Remark~\ref{remark: main thm remark}. Theorem~\ref{thm:SparseProj}, along with \eqref{eq: N0' defn}, relates the above roles to \gugan{the} choice of $\alpha$ and $\beta;$ it suggests several valuable tradeoffs between speeding up convergence of the noiseless ODE, and mitigating the martingale noise \gugan{and other drift factors (see Subsection~\ref{subsec:Comparison})} to aid the SA to follow this process. Namely, $N_0'$ explodes:}
\begin{enumerate}
\item \gal{As $\alpha$ or $\beta$ approach $0$ (stepsizes approach constants); this stems from $N_0$ blowing up. This occurs} since the stepsizes' slow decay rate impairs i) their ability to mitigate the martingale noise \gugan{and other drift factors}; and hence ii) the ability of the SA to \gugan{track} the ODE trajectories.

\item \gal{As $\alpha$ and $\beta$ get close to each other; this is due to $N_0$ blowing up. This occurs} \gugan{as} the \gugan{true} two-timescale nature is \gugan{then} nullified (see Remark~\ref{rem: true 2TS}).  Our analysis suggests that convergence of $z_n$ to $\zS$ \gugan{must} be faster than that of $\theta_n$ to $\thS$, \gugan{and a decaying stepsize ratio $\eta_n$ ensures this.}

\item \gal{As $\alpha$ or $\beta$ approach $1$, the largest value for which \eqref{eqn:stepSizeBound} \gugan{holds}; this stems from $N_1$ blowing up. This occurs} \gugan{as} the stepsizes then decay too fast, impairing the speed of the ODE convergence; more accurately, $N_1$ (see Table~\ref{tab: epsilon dependent constants}) moves away from exponential nature to inverse polynomial.
\end{enumerate}

\section{\gal{Proof Outline of Theorem~\ref{thm:condMain}}}
\label{sec: analysis outline}

\gal{In this section, we bring the essence of the proof of Theorem~\ref{thm:condMain}.} For \gal{intermediate results and the complete proof}, see Appendix~\ref{sec:SupMat}.
\gal{Naturally, t}hroughout \gugan{this} section, we assume \eqref{eq: Gn' definition}. Also, as mentioned in Section~\ref{sec:2TSSetup}, we work \gugan{here} with the iterates $\{z_n\}$ defined using \eqref{eq:z_iter} instead of $\{w_n\}$ directly. As stated in Remark~\ref{rem:AnalysisApproach}, our analysis follows through thanks to this choice.

%We \gal{now give a \gugan{summary}} of our proof for Theorem~\ref{thm:condMain}. All proofs for the results in this section are given in Appendix~\ref{sec:SupMat}.
%\subsection*{\gal{Analysis Outline Summary}}
%
The proof has two key steps. First, in Subsection~\ref{subsec:Comparison}, we use the Variation of Parameters (\vop) formula \citep{lakshmikantham1998method} to quantify the distance
between the \gal{SA} trajectories \gal{generated with \eqref{eq:theta_iter} and \eqref{eq:z_iter} and} suitable solutions of their respective limiting ODEs \eqref{eqn:tLimODE} and \eqref{eqn:zLimODE}.
%of the perturbed trajectories $\bart{(t)}$ and $\barz{(s)}$ from their unperturbed trajectories $\theta(t) \equiv \tSol{t}$ and $z(s) \equiv \zSol{s},$ respectively.
%This is done by splitting the perturbations into three parts per each trajectory, as described in Section~\ref{subsec:Comparison}.
%Using this, we then obtain upper bounds on $\norm{\bart(t) - \theta(t)}$ and $\norm{\barz(s)- z(s)}.$

As for the second step, note that the choice of $N_0$ in Theorem~\ref{thm:condMain}
%and Table~\ref{tab: epsilon dependent constants},
ensures that $\{\beta_k\}_{k \geq N_0}$ and $\{\eta_k\}_{k \geq N_0}$ are sufficiently small---i.e., of order $O(\max(\et,\ez))$.
%From the statement of Theorem~\ref{thm:condMain} and Table~\ref{tab: epsilon dependent constants}, observe that $N_0$ is such that the stepsizes $\{\beta_k\}_{k \geq N_0}$ and $\{\eta_k\}_{k \geq N_0}$ are sufficiently small.
Exploiting this fact and using the Azuma-Hoeffding %martingale concentration
inequality, in Subsection~\ref{subsec:conc_bound_lemmas}
we show that the bounds on the distances obtained in the first step are small with very high probability.
%$t \geq \tI{n_0}$ and $s \geq \sI{n_0}.$}
More explicitly, when the ODE solutions %$\theta(t)$ and $z(s)$
are sufficiently close to $\thS$ and $\zS$ respectively, \gugan{we show that} the same also holds for
the sequences $\{\theta_n\}$ and $\{z_n\}$
%$\bart(t)$ and $\barz(s)$
with high probability.
%\gal{for all $n \geq {n_0}.$}
A visualization of the process is given in Fig. \ref{fig:trajectory}.
%; see Section~\ref{subsection: analysis preliminaries} for the extraneous notations.

Before discussing these two steps, we now introduce some notations and terminology.

\subsection{Analysis Preliminaries}
\label{subsection: analysis preliminaries}

\gugan{To begin with,} we define the linearly interpolated trajectories of the iterates $\{\theta_n\}$ and $\{z_n\}.$ Having a continuous version of the discrete SA algorithm enables our analysis. Keeping \eqref{eq: time definition} in mind, let $\bart(\cdot)$ be the linear interpolation of $\{\theta_n\}$ on $\{t_n\};$ i.e., let $\bart(t_n) = \theta_n$ and, for $\tau \in (t_n, t_{n + 1}),$ let
\begin{equation}
\label{eqn:xLinearInterpolation}
\bart(\tau) = \bart(t_{n}) + \tfrac{(\tau - t_n)}{\st_n} [\bart(t_{n + 1}) - \bart(t_{n})].
\end{equation}
\gugan{Similarly,} let $\barz(\cdot)$ be the linear interpolation of $\{z_n\},$ but on  $\{s_n\}.$ For $\tau \in [\tI{n}, \tI{n + 1}),$ let
\begin{equation}
\label{defn:xi}
\xi(\tau) := s_n + \tfrac{\beta_n}{\alpha_n}(\tau - t_n).
\end{equation}
The mapping $\xi(\cdot)$ linearly interpolates $\{\sI{n}\}$ on $\{\tI{n}\}.$
%Let $\et,~\ez,~\Rti,~\Rzi > 0$ be such that
%\begin{equation}
%\et < \Rti ,~ \ez < \Rzi, ~\norm{\theta_0 - \thS} \leq \Rti,~\norm{z_0} \leq \Rzi. \label{eq: G_n' definition}
%\end{equation}

\gugan{
With the first parameter being time, the second being starting time, and the third being initial point, let $\tSol{t},$ $t \geq \tI{n_0},$ be the solution to \eqref{eqn:tLimODE} satisfying $\tSol{\tI{n_0}} = \theta_{n_0}.$ Similarly, define $\zSol{s}.$
} From \eqref{eqn:tLimODE} and standard ODE results {(see \citep[p. 129]{hirsch2012differential})},
\begin{equation}
\label{eqn:tOdeSol}
\theta(t)\equiv
\tSol{t} = \thS + e^{-\Xt(t - \tI{n_0})} (\theta_{n_0} - \thS), \quad \forall t \geq \tI{n_0}.
\end{equation}
In the same way, it follows from \eqref{eqn:zLimODE} that
\begin{equation}
\label{eqn:zOdeSol}
z(s)
\equiv
\zSol{s} = e^{- \Ww (s - \sI{n_0})} z_{n_0}, \quad \forall s \geq \sI{n_0}.
\end{equation}

\gugan{
\begin{remark}
\label{rem:MonDec}
Since $\Xt$ is positive definite due to \ref{assum:posDef}, \eqref{eqn:tOdeSol} implies that $\lim_{t \to \infty} \tSol{t} = \thS.$ Further, $\frac{d}{dt}\|\theta(t) - \thS\|^2 = -2(\theta(t) - \thS)^\top \Xt (\theta(t) - \thS)<0$; hence, assuming \eqref{eq: Gn' definition} holds,
$
\norm{\tSol{t} - \thS} \leq \Rti, \quad \forall t \geq \tI{n_0}.
$
Likewise, we have $\lim_{s \to \infty} \zSol{s} = \zS$  and $\norm{\zSol{s}} \leq \Rzi$ for all $s \geq \sI{n_0}.$
\end{remark}
}

%\gugan{Our key idea is to suitably compare} the SA trajectories $\bart(t)$ and $\barz(s)$ with the limiting ODE trajectories $\theta(t) \equiv \theta(t, \tI{n_0}, \bar{\theta}(\tI{n_0})) = \theta(t, \tI{n_0}, \theta_{n_0})$ and $z(s) \equiv z(s, \sI{n_0}, \bar{z}(\sI{n_0})) = z(s, \sI{n_0}, z_{n_0}),$ respectively.  \gal{We chose a large enough  $n_0$ for the stepsizes to be small enough to mitigate the noise, for controlling perturbation of the above SA trajectories away from the limiting ODE behavior.}

%Proving \gal{Theorem~\ref{thm:condMain}} is done in two steps.

%	\begin{figure}
%		\floatbox[{\capbeside\thisfloatsetup{capbesideposition={right,top},capbesidewidth=4cm}}]{figure}[\FBwidth]
%		{\caption{A test figure with its caption side by side}\label{fig:test}}
%		{\includegraphics[width=5cm]{name}}
%	\end{figure}
%
\begin{figure}
\floatbox[{\capbeside\thisfloatsetup{capbesideposition={right,center},capbesidewidth = 42ex}}]{figure}[\FBwidth]
{\caption{Visualization of the proof methodology. The red SA trajectories $\{\theta_n\}$ and $\{z_n\}$ are compared  to their blue respective limiting ODE trajectories $\theta(t)$ and $z(s)$. The three balls on each side of the figure (from small to large), are respectively the solution's $\epsilon$-neighborhood; the $R^{\text{in}}$ ball in which the SA trajectory and ODE trajectory are initialized; and the $R^{\text{out}}$ ball in which the SA trajectory is ensured to reside.
%		  The mapping from the $\theta$ time domain $t_n$ to the $z$ time domain $s_n$, as given in \eqref{defn:xi}, is visualized as well.
	  } \label{fig:trajectory} }
	{%\begin{center}
\hspace{-4ex}
		\includegraphics[scale=0.16]{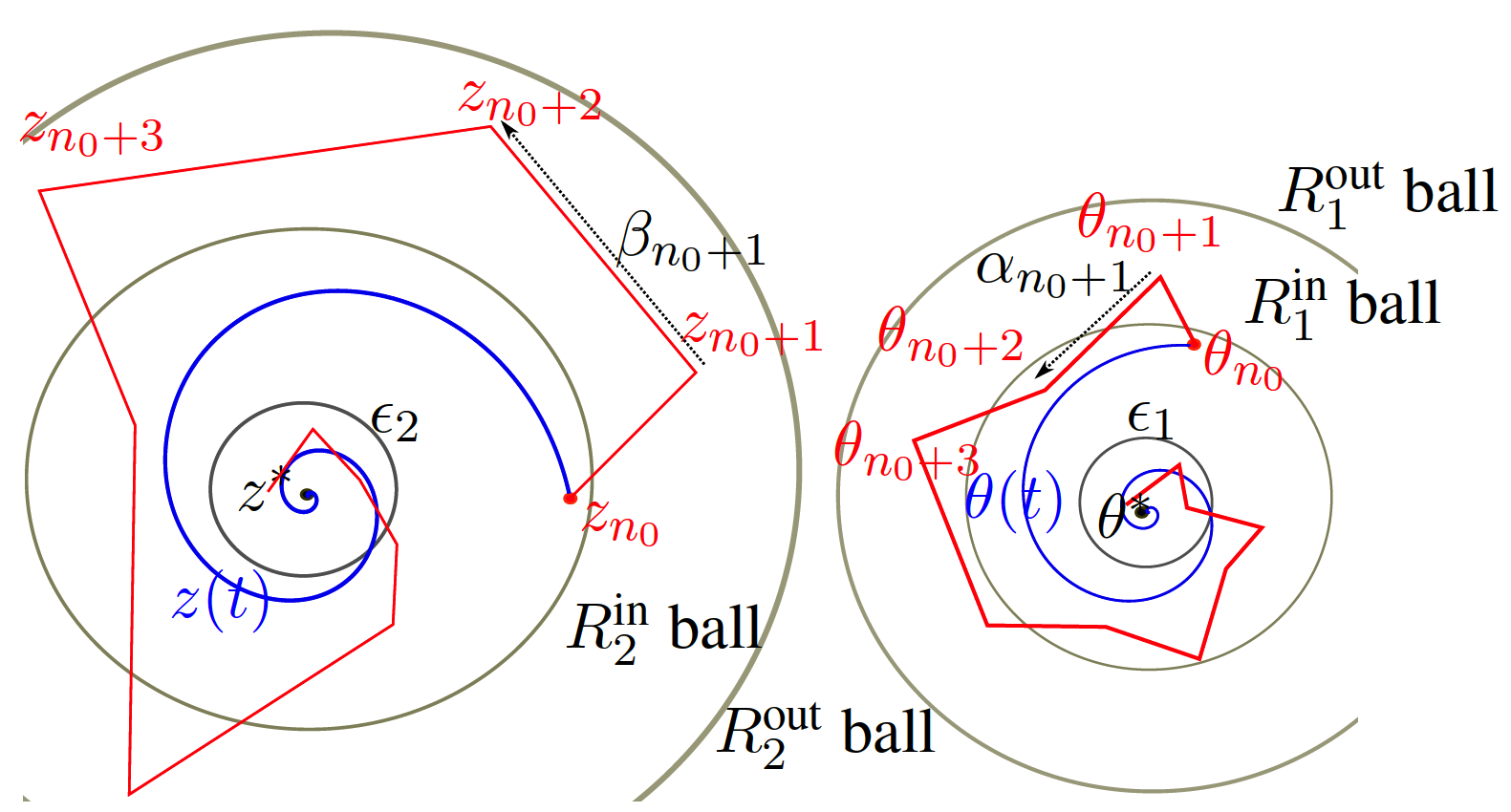} \hspace{-2ex}}
	%\end{center}}
\end{figure}

\subsection{Comparing the SA and corresponding Limiting ODE Trajectories}
\label{subsec:Comparison}

Our aim here is to use the \vop\ formula to bound $\norm{\barz(s) - z(s)}$ and $\norm{\bart(t) - \theta(t)}$. Note that both the SA trajectory $\bart(t)$ and the corresponding limiting ODE trajectory $\theta(t)$ equal $\theta_{n_0}$ at time $t = \tI{n_0}.$ Similarly, $\barz(\sI{n_0}) = z(\sI{n_0}) = z_{n_0}.$

Using $\eqref{eq: z_n def},$  \eqref{eq:theta_iter} translates to
$
\theta_{n + 1} = \theta_n + \alpha_n[\bt - \Xt \theta_n] + \alpha_n [-\Wt z_n] + \alpha_n \Mt_{n + 1} \enspace.
$
Iteratively using the above update rule, we then have
\[
\theta_{n + 1} = \theta_{n_0} + \sum_{k = n_0}^{n} \alpha_k [\bt - \Xt \theta_k - \Wt z_k + \Mt_{k + 1}] \enspace.
\]
From this and the definition of $\bart(\cdot),$ it consequently follows that
\begin{equation}
\label{eqn:tTotUpd}
\bart(t) = \theta_{n_0} + \int_{\tI{n_0}}^{t}[\bt - \Xt \bart(\tau)] \df \tau + \int_{\tI{n_0}}^{t} \zeta(\tau) \df \tau, \enspace \forall t \geq \tI{n_0} \enspace,
\end{equation}
with $\zeta(\tau) := \zetD(\tau) + \zetT(\tau) + \zetM(\tau)$, %denotes the total perturbation at time $\tau.$ For
where, for $\tau \in [\tI{k}, \tI{k + 1}),$
$
\zetD(\tau) := \Xt [\bart(\tau) - \theta_k], \quad   \zetT(\tau) := -\Wt z_k, \quad \text{and} \quad \zetM(\tau) := \Mt_{k + 1} \enspace.
$
Let $\EtD(t) = \int_{\tI{n_0}}^{t}e^{-\Xt(t -\tau)} \zetD(\tau) \df \tau$.
Define
$\EtT(t)$ and $\EtM(t)$ in the same spirit.
%the tracking error as $\EtT(t) = \int_{\tI{n_0}}^{t}e^{-\Xt(t -\tau)} \zetT(\tau) \df \tau$, and the martingale difference noise as $\EtM(t) = \int_{\tI{n_0}}^{t}e^{-\Xt(t -\tau)} \zetM(\tau) \df \tau$.
%
We refer to these three terms as the
%These three terms as the respectively correspond to the perturbations due to
discretization error, tracking error, and martingale difference noise%
, respectively.
The tracking error is called so, because it depends on $z_k = w_k-\lambda(\theta_k)$ which, by \eqref{eq:z_iter}, tells how close $w_k$ is to its ODE solution $\lambda(\theta_k)$.
%from \eqref{eq:z_iter}, is the difference between $w_k$ and $\lambda(\theta_k).$
%These definitions will become clear below in \eqref{eqn:TrajComp}.
%
From \eqref{eqn:tLimODE}, we have $\theta(t) = \theta_{n_0} + \int_{\tI{n_0}}^{t} [\bt - \Xt \theta(\tau)]\df \tau,$ and thus \eqref{eqn:tTotUpd} can be viewed as a perturbation of $\theta(t).$
Defining then $\Et(t) := \EtD(t) + \EtT(t) + \EtM(t)$, and applying the \vop\ formula \gal{(see Appendix~\ref{sec:VoC})}, it follows easily that
\begin{equation}
\bart(t) = \tSol{t} + \Et(t) \enspace. \label{eqn:TrajComp-t}
\end{equation}

Using \eqref{eq:z_iter}, it is easy to see in the same way {as above} that
%In the same way, it is easy to see from \eqref{eq:z_iter} that
%
\begin{equation}
\label{eqn:zTotUpd}
\barz(s) = z_{n_0} + \int_{\sI{n_0}}^{s} [-\Ww] \barz(\mu) \df \mu + \int_{\sI{n_0}}^{s} \chi(\mu) \df \mu, \enspace \forall s \geq \sI{n_0} \enspace,
\end{equation}
with $\chi(\mu) := \chizD(\mu) + \chizS(\mu) + \chizM(\mu)$, where, for $\mu \in [\sI{k}, \sI{k + 1}),$
\begin{equation}
\chizD(\mu) := \Ww[\barz(\mu) - z_k], \quad \chizS(\mu) := \frac{\lambda(\theta_k) - \lambda(\theta_{k + 1})}{\beta_k}, \quad \chizM(\mu) := \Mw_{k+1} \enspace.
\end{equation}
% denotes the total perturbation at time $\mu.$
Let $\EzD(s) = \int_{\sI{n_0}}^{s} e^{-\Ww(s - \mu)} \chizD(\mu) \df \mu$.
Define $\EzS(t)$ and $\EzM(t)$ in the same spirit.
%the tracking error as $\EtT(t) = \int_{\tI{n_0}}^{t}e^{-\Xt(t -\tau)} \zetT(\tau) \df \tau$, and the martingale difference noise as $\EtM(t) = \int_{\tI{n_0}}^{t}e^{-\Xt(t -\tau)} \zetM(\tau) \df \tau$.
%
We refer to these three terms as
discretization error, slow drift in the equilibrium of \eqref{eqn:wLimODE}, and martingale difference noise.
%The three terms $\chizD(\cdot), \chizS(\cdot),$ and $\chizM(\cdot)$ respectively correspond to perturbations due to discretization error, slow drift in the equilibrium of \eqref{eqn:wLimODE}, and martingale difference noise.
%
%To clarify on the role of
We refer to $\chizS(\mu)$ as the slow drift error because as $\{\theta_n\}$ evolve, the ODE solution $\{\lambda(\theta_n)\}$ drift,
% the drift is due to the fact that $\theta_n$ evolves
and it is slow since $\{\theta_n\}$ is updated on the slow time scale $\{\tI{n}\}$ (recall that $\eta_n \to 0$).
Finally, defining $\Ez(t) := \EzD(t) + \EzS(t) + \EzM(t)$, it follows simlarly as above that
\begin{equation}
\barz(s) = \zSol{s} +  \Ez(s) \enspace. \label{eqn:w-TrajComp}
\end{equation}

%From \eqref{eqn:tLimODE}, we have $\theta(t) = \theta_{n_0} + \int_{\tI{n_0}}^{t} [\bt - \Xt \theta(\tau)]\df \tau.$ Therefore, \eqref{eqn:tTotUpd} can be viewed as a perturbation of $\theta(t).$ Similarly,  \eqref{eqn:zTotUpd} can be viewed as a perturbation of $z(s).$ Applying the \vop\ formula, it is then easy to see that
%%
%\begin{equation}
%%
%\bart(t) = \tSol{t} + \Et(t) \enspace, \enspace\text{ and }\enspace \barz(s) = \zSol{s} +  \Ez(s) \enspace, \label{eqn:TrajComp}
%%
%\end{equation}
%%
%where $\Et(t) := \EtD(t) + \EtT(t) + \EtM(t)$
%%with $\EtD(t) = \int_{\tI{n_0}}^{t}e^{-\Xt(t -\tau)} \zetD(\tau) \df \tau$ and so on,
%and $\Ez(s) := \EzD(s) + \EzS(s) + \EzM(s)$.
%%with $\EzD(s) = \int_{\sI{n_0}}^{s} e^{-\Ww(s - \mu)} \chizD(\mu) \df \mu$ and so on.
The below result is now trivial to see.

%\int_{\tI{n_0}}^{t} e^{-\Xt(t - \tau)} \zeta(\tau) \df \tau

% \int_{\sI{n_0}}^{s}e^{-\Ww(s - \mu)} \chi(\mu) \df \mu

\begin{lemma}
\label{lem:PerBd}
The following two statements hold:
\begin{enumerate}
\item For $t \geq \tI{n_0},$  $\norm{\bart(t) - \tSol{t}}\leq \norm{\EtD(t)} + \norm{\EtT(t)} + \norm{\EtM(t)} \enspace.$
\item For $s \geq \sI{n_0},$ $\norm{\barz(s) - \zSol{s}}\leq \norm{\EzD(s)} + \norm{\EzS(s)} + \norm{\EzM(s)} \enspace.$
\end{enumerate}
\end{lemma}

%\begin{remark}
To stress the tightness of the above analysis, we compare it with that in \citep[p. 14]{borkar2008stochastic}. There, the distance between the SA and ODE trajectories is bounded by the tail sum of the squared stepsizes; this necessitates the usual square summability assumption. We do not require it here thanks to the additional exponentials, $e^{-\Xt(t - \tau)}$ and $e^{-\Ww(s - \mu)},$ in the error terms $\EtD(t), \EzD(s),$ etc., which is a consequence of the \vop\ formula.

\subsection{Concentration Bounds for Two-Timescale SA} \label{subsec:conc_bound_lemmas}

Next, with Lemma~\ref{lem:PerBd} bounding the distance of $\bart(t)$ and $\barz(s)$ from their respective ODE trajectories for all $t$ and $s$, we consequently bound the distance of $\bart(t)$ and $\barz(s)$ from the solutions $\thS$ and $\zS$. To do so, we break the convergence event into an incremental union using a novel inductive technique (see Appendix~\ref{subsec:SuperSet}, Lemma~\ref{lem:IntEv}). Each event in the union has the following structure: ``good'' up to time $n$ (ensured by an event $G_n$, where the iterates remain bounded in certain regions) and ``bad'' in the subsequent interval ($\bart(t_{n+1})$ and $\barz(s_{n+1})$ leave the bounded regions).  By conditioning on $G_n,$ and using \eqref{eq: Gn' definition} with Lemma~\ref{lem:PerBd}, we bound $\norm{\bart(t) - \thS}$ and $\norm{\barz(s) - \zS}.$  Each of the resulting bounds consists of three kinds of terms (see Appendix~\ref{appendix: perturbation bounds},  Lemmas~\ref{lem:rzBd} and \ref{lem:rtBd}): i) sum of martingale differences (originating in $\EiM$), ii) {stepsize based term} (originating in $\EiD,~\EtT,~\EzS$), and iii) {exponentially decaying term} (originating in the ODE trajectory convergence).
%For large enough $n$, type i) terms are small with high probability due to \ref{assum:Noise} and the Azuma-Hoeffding martingale concentration inequality;  type ii) terms are small for sufficiently large  $n$; type iii) terms are small for \gal{n sufficiently larger than $n_0$.}
{
Type i) terms are small w.h.p. due to
%are small with high probability due to \ref{assum:Noise} and
the Azuma-Hoeffding inequality; these terms give the r.h.s. in \eqref{eqn:ProbBound}. Type ii) terms are small since $N_0$ is chosen sufficiently large (consult Table~\ref{tab: constants} for the definition of $N_0$). Type iii) terms are small for \gal{n sufficiently larger than} $N_0$ (in particular, for $n>N_1$---consult Table~\ref{tab: constants} for the definition of $N_1$).
This summarizes the proof of \gal{Theorem}~\ref{thm:condMain}, which is described in Appendix~\ref{sec: conc_bound_appendx}.

\section{Applications to Two-timescale Reinforcement Learning}
\label{sec:Appl}
Here we show how our Theorem~\ref{thm:SparseProj} implies convergence rates of linear two-timescale methods for policy evaluation in Markov Decision Processes (MDP). An MDP is defined by the 5-tuple $(\mathcal{S}, \mathcal{A},P,R,\gamma)$ \citep{sutton1988learning}, \gal{where these are respectively the state and action spaces, transition kernel, reward function, and discount factor.}
% ${\mathcal S}$ is the set of states, ${\mathcal A}$ is the set actions, $P = P(s'|s,a)$ is the transition kernel, $R(s,a,s')$ is the reward function, and $\gamma\in(0,1)$ is the discount factor.
%In each time-step, the process is in some state $s_n\in {\mathcal S}$, an action $a_n\in {\mathcal A}$ is taken, the system transitions to a next state $s_n'\in {\mathcal S}$ according to a transition kernel $P(s_n,a_n,s_n')$, and an immediate reward $r_n$ is received according to $R(s_n,a_n,s_n')$.
 Let policy $\pi:{\mathcal S} \rightarrow {\mathcal A}$ be a stationary mapping from states to actions and $V^\pi(s) = \mathbb{E}^\pi[\sum_{n=0}^\infty\gamma^nr_n|s_0 = s]$ be the value function at state $s$ w.r.t $\pi$.

We consider the policy evaluation setting. In it, the goal is to estimate the value function $V^\pi(s)$ with respect to a given $\pi$ using linear regression, i.e., $V^\pi(s)\approx\theta^\top\phi(s)$, where $\phi(s) \in \mathbb{R}^d$ is a feature vector at state $s$, and $\theta \in \mathbb{R}^d$ is a parameter vector. For brevity, we omit the notation $\pi$ and denote $\phi(s_n),~\phi(s_n')$  by $\phi_n,~\phi_n'$. Finally, let $\delta_n = r_n +\gamma\theta_n^\top\phi_n' - \theta_n^\top\phi_n~,A = \bE[\phi(\phi-\gamma\phi')^\top]~,C=\bE[\phi\phi^\top]$, and $b=\bE[r\phi]$, where the expectations are w.r.t. the stationary distribution of the induced chain \footnote{The samples $\{(\phi_n,\phi'_n)\}$ are generated iid. This assumption is standard when dealing with convergence bounds in reinforcement learning \citep{liu2015finite,sutton2009convergent,sutton2009fast}.  In the few papers where this assumption is not made, it is replaced with an exponentially-fast mixing time assumption \citep{korda2015td,tsitsiklis1997analysis}.}%\footnote{Indeed, $(\phi_n,\phi_n’)$ is sampled iid. We did not specify it since it is standard when dealing with convergence bounds in RL ((Liu et al.2015;Sutton et al.,2009a,b), etc.). The few papers that obviate this assumption assume some other strong properties to hold such as exponentially-fast mixing time (Korda \& Prashant,2015;Tsitsiklis \& V.Roy,1997). We shall add this remark in Section 5.
}.

We assume all rewards $r(s)$ and feature vectors $\phi(s)$ are bounded:
%\begin{align} \label{eqn:boundRewardsFeatures}
\(
|r(s)| \leq 1 ,	\|\phi(s)\| \leq 1 ~ \forall s \in S.
\)
%\end{align}
Also, it is assumed that the feature matrix $\Phi$ is full rank, so $A$ and $C$ are full rank. This assumption is standard \citep{maei2010toward,sutton2009convergent}. Therefore, due to its structure, $A$ is also positive definite \citep{bertsekas2012dynamic}. Moreover, by construction, $C$ is positive semi-definite; thus, by the full-rank assumption, it is actually positive definite.
%Also, it is known that both $A$ and $C$ are positive definite \citep{maei2010toward,bertsekas2012dynamic}.
%It is known that $A$ is positive definite \citep{}. Also, by construction, $C$ is  positive semidefinite; thus, by the full-rank assumption it is actually positive definite.

\begin{table*}[t]
	{%\small
\begin{center}
\begin{tabular}{|c|c|c|c|c|}
\hline
Method & $\Xt$ & $\Ww$ & $\mt$ & $\mw$ \\
\hline
& & & & \\ [-2ex]
GTD(0) & $A^\top A$ & $\Id$ & $(1+\gamma+\|A\|)$ & $ 1 + \max(\|b\|,\gamma+\|A\|)$ \\ [0.5ex]
GTD2 & $A^\top C^{-1} A$ & $C$ & $ (1+\gamma+\|A\|)$ & $ 1 + \max(\|b\|,\gamma+\|A\|,\|C\|)$ \\ [0.5ex]
TDC & $A^\top C^{-1} A$ & $C$ &  $ (2+\gamma+\|A\|+\|C\|)$ & $ (2+\gamma+\|A\|+\|C\|)$ \\
\hline
\end{tabular}
\end{center}
\caption{\label{table:RL algorithms} Translation of notations for relevant matrices and constants in the case of the GTD family of algorithms. The parameters $\Xt,~\Ww,~\mt,~\mw$ are defined in Section~\ref{sec:2TSSetup}.}
}
\end{table*}

\subsection{\gal{The GTD(0) Algorithm}}
\label{sec: GTD0}
We now present the GTD(0) algorithm \citep{sutton2009convergent}, verify its required assumptions, and obtain the necessary constants to apply Theorem~\ref{thm:SparseProj} for it. \gal{ GTD(0) is designed to minimize the objective function %(recall (\ref{eq: b-Atheta}))
$
J^{\rm NEU}(\theta)
%=
%\tfrac{1}{2}\|\bE[\delta_n(\theta)\phi]\|_2^2
=
\tfrac{1}{2}(b-A\theta)^\top (b-A\theta).
$
Its update rule is}
\begin{equation*}
\label{eq:GTD0 iter}
\theta_{n + 1} = \theta_n + \st_n \left(\phi_n - \gamma \phi_n'\right)\phi_n^\top w_n, \quad\quad\quad
w_{n + 1} = w_n + \sw_n r_n\phi_n + \phi_n[\gamma\phi_n'-\phi_n]^\top\theta_n .
\end{equation*}
\gal{It thus takes the form of} \eqref{eq:theta_iter} and \eqref{eq:w_iter} with
%\begin{align*}
$h_1(\theta,w) = %\bE\left[\left(\phi_n - \gamma \phi_n'\right)\phi_n^\top\right]w =
A^\top w \enspace,
h_2(\theta,w) = %\bE\left[ \delta_n\phi_n \right] -w =
b-A\theta - w ~,
\Mt_{n+1} = \left(\phi_n - \gamma \phi_n'\right)\phi_n^\top w_n - A^\top w_n \enspace, \Mw_{n+1} %=& \left(\delta_n\phi_n - w_n\right) - \left(b-A\theta - w_n\right)
%\\
%  =& r_n\phi_n + \phi_n[\gamma\phi_n'-\phi_n]^\top\theta_n - w_n \\
%  & - \left(b-A\theta - w_n\right) \enspace.
= r_n\phi_n + \phi_n[\gamma\phi_n'-\phi_n]^\top\theta_n  - \left(b-A\theta_n \right) \enspace.
$
That is, in case of GTD(0), the relevant matrices in the update rules take the form
$\Tt = 0$, $\Wt = -A^\top$, $v_1 = 0$, and $\Tw=A$, $\Ww=\Id$, $v_2=b$.
Additionally, $\Xt = \Tt - \Wt\Ww^{-1}\Tw = A^\top A$.
By our assumption above, both $\Ww$ and $\Xt$ are symmetric positive definite matrices, and thus the real parts of their eigenvalues are also positive.
Also,
$
\|\Mt_{n+1}\| \leq (1+\gamma+\|A\|) \|w_n\| ,
$
$
\|\Mw_{n+1}\|
\leq  1+\|b\| + (1+\gamma+\|A\|)\|\theta_n\|.
$
\gugan{Hence, \ref{assum:Noise} is} satisfied with constants $\mt = (1+\gamma+\|A\|)$ and $\mw = 1 + \max(\|b\|,\gamma+\|A\|)$.

We now can apply \gal{Theorem~\ref{thm:SparseProj}  for a specific stepsize choice to obtain the following simplified result. A more detailed statement with all relevant constants can be directly derived from Theorem~\ref{thm:SparseProj}. }

\gal{
\begin{corollary}[Convergence Rate for Sparsely Projected GTD(0)] \label{cor: RL}
	Consider the Sparsely Projected variant of GTD(0) as in \eqref{eq:thetap_iter} and \eqref{eq:wp_iter}. Set some $\kappa \in (0,1).$ Then for  $\st_n = 1/n^{1-\kappa}$, $\sw_n = 1/n^{(2/3)(1-\kappa)},$ the algorithm converges at a rate of $O(n^{-1/3+\kappa/3})$ w.h.p.
\end{corollary}
}
 \gal{For GTD2 and TDC \citep{sutton2009fast}, the above result can be similarly be reproduced; see Table~\ref{table:RL algorithms} for the relevant parameters.} The detailed derivation is provided in Appendix~\ref{sec: RL appendix}.
 
  A reviewer has pointed us to the fact that, unlike in the GTD(0) and GTD2 convergence results, there exists a special condition on the stepsize ratio for TDC \citep[Theorem~3]{maei2011gradient}. However, we find that this condition to be unnecessary because $A$ and $C$ are positive definite.

\section{Discussion}
In this work, we conduct the first finite sample analysis for two-timescale SA algorithms. We provide it as a general methodology that applies to all linear two-timescale SA algorithms.

A natural extension to our methodology is considering the non-linear function-approximation case, in a  similar fashion to \citep{thoppe2015concentration}. Such a result can be of high interest due to the recently growing attractiveness of neural networks in the RL community. An additional direction for future research is to extend our results to actor-critic RL algorithms.  Moreover, off-policy extensions can be made for the results here; see Appendix~\ref{sec: off policy}. Lastly, for a discussion on the tightness of the results here and comparison to known asymptotic rates see Appendix~\ref{sec: tightness}.

\acks{This research was supported
%	by the European Community's
%	Seventh Framework Programme {(FP7/2007-2013)}
%	under grant agreement
	by ERC grant 306638 (SUPREL). GT was initially supported by ERC grant 320422 (at Technion) and now by NSF grants DMS-1613261, DMS-1713012, IIS-1546413.}

\bibliographystyle{plain}
\bibliography{2TS_References}

\newpage

\appendix

\section{Off-Policy Extensions}
\label{sec: off policy}
Off-policy results play a central role in reinforcement learning; however, we were focusing here exclusively on the on-policy setting. Nonetheless, our results can be similarly extended as in \citep{liu2015finite}. Namely, we can repeat the elegant reduction conducted there, where the bound on $\|\theta_n-\theta^*\|$ is transformed into one on the approximation error  $\|V-\bar{v}_n\|=\|V-\Phi\bar{\theta}_n\|.$ More precisely, we can bound the second term on the RHS in \citep[Appendix~B, (42)]{liu2015finite} using \citep[Theorem~2]{kolter2011fixed}, and apply our result to bound the first one. Except for a slightly different rescaling of the matrices (since we use L2 norm as opposed to $\xi$-weighted L2), we would then obtain an off-policy result as in Proposition~5 there. Two benefits would then be: a result directly consisting of $\theta_n$ (instead of its average), and a generic stepsize family $n^{-\alpha}$ (instead of $C/\sqrt{n}$, where $C=f(\|A\|+\|b\|),$ as depicted above \citep[Appendix~B, (40)]{liu2015finite}. Notice, also, that transforming one type of bound to the other, as explained above, is a trick by \citep{liu2015finite} that can be applied in general and not only in our case.

\section{Tightness}
\label{sec: tightness}
Here, we compare our convergence rates with other existing works. To the best of our knowledge, no other finite time results exist for two-timescale SA algorithms. However, there are a few relevant works that deal with this question in an asymptotic sense. Before discussing them, we highlight some key differences between finite-time rates and aysmptotic ones. In the latter, the constants hidden in the order notations are often sample-path dependent and hence are less attractive to practitioners. Contrarily, explicit constants in finite-time rates, as ours, often reveal intriguing dependencies amongst system and stepsize parameters that crucially affect convergence rates (e.g., $1/q_i$ in Table~\ref{tab: constants}; see also \citep[Section~6]{dalal2018finite}). Moreover, trajectory-independent constants help in obtaining stopping time theorems.

Following Remark~\ref{rem: optimal rate}, the best convergence rate possible according to our results is $n^{-1/3}.$ This contrasts the single time-scale case, where the optimal rate is known to be $n^{-1/2}$ under various settings.
In the context of asymptotic rates, there exist two works that deal with two-timescale SA which achieve the optimal rate of $n^{-\alpha/2}$ for the slow-timescale iterate and $n^{-\beta/2}$ for the fast-timescale iterate \citep{konda2004convergence,mokkadem2006convergence}. In \citep{konda2004convergence}, according to Assumption~2.1, the noise sequence is assumed to be independent of itself, and their variance-covariance matrices are constant w.r.t. iteration index $n$. In our case, in contrast, the noise depends on $(\theta_n,w_n)$, making the variance-covariance matrices of the noise sequence explicitly depend on $n$. These differences make their results inapplicable for the RL algorithms we consider in our paper; see Section~\ref{sec: GTD0}. A later work \citep{mokkadem2006convergence} improved upon \citep{konda2004convergence} by removing the above assumption. There, in (A1), convergence was posed as an assumption on its own. Such an assumption is not straighforward to verify in general; it was only recently established for square-summable stepsizes \citep{lakshminarayanan2017stability}. However, in the case of non-square-summable stepsizes (which is not covered in \citep{mokkadem2006convergence}) this Assumption (A1) has not been shown to hold in general, since converegence is an open question for such stepsizes.

Lastly, while we do not show our bound to be tight, we stress that our result coincides with known results on a particular SA method of two-timescale nature, called Spall's method \citep[Proposition 2]{spall1992multivariate} and \citep[Theorem~5.1]{gerencser1997rate}. Specifically, it was shown for the iterate $\theta_n$ there that $n^{-\kappa}\theta_n$ converges in distribution to some normal distribution for various parameter settings that restrict $\kappa$ to be at least $1/3$.
This raises the intriguing question whether the rates achieved in our work and in \citep{spall1992multivariate,gerencser1997rate} are sub-optimal and stem from loose analyses, or whether it is the problem setup itself that intrinsically limits the rate to $n^{-1/3}$.

\section{A Bound for Sub-exponential Series}
\label{sec: bound for subexponential series}

\gugan{Let $p \in (0, 1)$ and $\hat{q} > 0.$ Let $i_1 \equiv i_1(p, \hat{q}) \geq 1$ be such that $
	e^{-\hat{q}\sum_{k = 1}^{n - 1} (k + 1)^{-p}} \leq  n^{-p}
	$ for all $n \geq i_1;$
	such an $i_1$ exists as the l.h.s. is exponentially decaying. Let
	\begin{equation}
	\label{Defn:Kg}
	K_g \equiv K_g(p, \hat{q}) := \max_{1 \leq i \leq i_1} i^{p} e^{-\hat{q} \sum_{k= 1}^{i - 1}(k + 1)^{-p}}.
	\end{equation}
}

\begin{lemma}[\gal{Closed-form \gugan{sub-exponential} bounds}]\label{lem: B formula}
	Let \gugan{$n_0 \geq 1$,} $B>0,$ and \gugan{$p \in (0,1).$} Then, for every $\kappa \in (0,1),$
	\begin{equation}
	\label{eq: B formula}
	\sum_{n = n_0}^{\infty} \exp[- B n^{p}]  \leq  \frac{2}{{B (1-\kappa) p}} \left[\frac{( 1- p)}{B\kappa {p}}\right]^{\frac{1 - {p}}{{p}}}
	\exp\left[B(2-\kappa)-\frac{(1-p)}{p} -B(1-\kappa) {n_0}^{p}\right] \enspace.
	\end{equation}
	\gugan{Further, for any $c, \hat{q} >0,$ and $n_0 \geq 1$, with $c_n := \sum_{k=0}^{n-1} [k+1]^{-2p}e^{-2\hat{q}\sum_{i=k+1}^{n-1}[i+1]^{-p}}$, we have
		\begin{equation}
		\sum_{n \geq n_0} \exp\!\left[\tfrac{-c \epsilon^2}{c_n}\right]
		\leq \tfrac{c_7}{\epsilon^{2/p}}
		e^{c_5\epsilon^2} e^{- c_6 \epsilon^2 n_0^{p}}
		\enspace,
		\label{eq: pre-delta bound}
		\end{equation}
		where
		\(
		c_7 \equiv c_7(c,\kappa,p,\hat{q})
		= 2 \left[\frac{K_g(p, \hat{q}) e^{\hat{q}}}{c \hat{q}}\right]^{1/p} \frac{1}{(1-\kappa) p^{1/p}} \left[\tfrac{1-p}{e \kappa}\right]^{\frac{1 - {p}}{p}}
		\),
		\(
		c_5 \equiv c_5(c, \kappa, p, \hat{q}) = \frac{c \hat{q} (2 - \kappa)}{K_g(p, \hat{q}) e^{\hat{q}}},
		\)
		and
		\(
		c_6 \equiv c_6(c, \kappa, p, \hat{q}) = \frac{c \hat{q} (1 - \kappa)}{K_g(p, \hat{q}) e^{\hat{q}}}.
		\)
	}
\end{lemma}

\begin{proof}
	\gugan{Let $\lfloor \cdot \rfloor$ denote the floor operation. Then, for $p \in (0, 1)$ and integers $n, i \geq 0,$ we have
		\begin{align}
		&|\{n :  \lfloor n^{{p}}\rfloor = i\}|
		\\ = & |\{n: i \leq n^p < i + 1\}| \nonumber
		\\ = & |\{n: i^{1/p} \leq n < (i + 1)^{1/p}\}| \nonumber
		\\ \leq & (i+1)^\frac{1}{p} - i^\frac{1}{p} + 1 \nonumber
		\\ \leq & 2\left[(i+1)^\frac{1}{p} - i^\frac{1}{p}\right],\label{eq: bounding the inteval length}
		\end{align}%
		where the last inequality follows since $(i+1)^\frac{1}{p} - i^\frac{1}{p} \geq 1$.
	}
	% the "-1" is necessary:
	% setting p=2/3 we obtain (5^{3/2}-4^{3/2}) < (11.19-8) < 4, whereas
	% |{n: n^p = 4}| = |{8, 9, 10, 11}|
	
	From the concavity of $x^p$, $x^p \leq x_0 ^p + \frac{d}{dx}(x^p)\big|_{x=x_0}(x-x_0)$ for all $x,x_0 \in \mathbb{R_+}.$ Equivalently,
	$$
	x_0 - x \leq (x_0^p - x^p)\left[\frac{d}{dx}(x^p)\big|_{x=x_0}\right]^{-1}.
	$$
	\gugan{Setting $x_0=(i+1)^{\frac{1}{p}}$ and $x=i^{\frac{1}{p}}$, it follows from \eqref{eq: bounding the inteval length} that}
	\begin{align}
	&|\{n :  \lfloor n^{{p}}\rfloor = i\} | \nonumber\\ \leq & 2\left[(i+1)^\frac{1}{p} - i^\frac{1}{p}\right] \nonumber\\ \leq & 2\left[\left((i+1)^\frac{1}{p}\right)^p - \left(i^\frac{1}{p}\right)^p\right]\left[px^{p-1}\big|_{x=(i+1)^{\frac{1}{p}}}\right]^{-1}  \nonumber\\ = &\frac{2}{p}(i + 1)^{\frac{1- p}{p}} \enspace. \label{calculus bound}
	\end{align}
	
	For any $\kappa \in (0,1)$, \gugan{observe that}
	$e^{ - x B\kappa} (x + 1)^{\frac{1 - p}{p}},$ \gugan{restricted to $x \geq 0,$} has a \gugan{global maximum} at $x=\frac{(1-p)}{B\kappa p}-1,$ and so
	\begin{equation}
	\label{eq: bounding the max term}
	\max_{i \geq 0} e^{ - i B\kappa} (i + 1)^{\frac{1 - p}{p}}
	\! \leq \! \left[\frac{1- p}{B\kappa p}\right]^{\frac{1 - {p}}{{p}}} \! \! \! e^{[\kappa B-\frac{(1-p)}{p}]} \enspace.
	\end{equation}
	
	\gugan{
		Fix an arbitrary $\kappa \in (0, 1).$ From the above observations, we get}
	\begin{align}
	& \sum_{n = n_0}^{\infty} \exp[- B n^{p}] \notag\\
	& \leq  \sum_{i = \lfloor n_0^{p}\rfloor }^{\infty} e^{-i B} |\{n :  \lfloor n^{{p}}\rfloor = i\} | \notag\\
	& \leq  \frac{2}{p}  \sum_{i = \lfloor n_0 ^{p}\rfloor }^{\infty} e^{-i B} \; (i + 1)^{\frac{1- {p}}{{p}}} \label{eqn:cardUB} \\
	& = \frac{2}{{p}}  \sum_{i = \lfloor n_0^{p}\rfloor }^{\infty} e^{-i B(1-\kappa)}\; e^{-i B\kappa}  \; (i + 1)^{\frac{1- {p}}{{p}}} \notag\\
	& \leq \frac{2}{{p}} \left[\frac{1 - p}{B\kappa p}\right]^{\frac{1 - {p}}{{p}}} \! \! e^{[\kappa B- \frac{(1-p)}{p}]} \hspace{-.5em} \sum_{i = \lfloor n_0^{p}\rfloor }^{\infty} \! \! e^{-i B(1-\kappa)}
	\label{eqn:MaxEst}\\
	& { \leq \frac{2}{{p}} \left[\frac{1 - p}{B\kappa p}\right]^{\frac{1 - {p}}{{p}}} \! \! e^{[\kappa B- \frac{(1-p)}{p}]} \hspace{-.5em} \int_{\lfloor n_0^p \rfloor - 1}^{\infty} e^{-x B (1 - \kappa)} \textnormal{d} x} \label{eqn:RightRSum} \\
	& \leq  \frac{2}{{B (1-\kappa) p}} \left[\frac{1- p}{B\kappa {p}}\right]^{\frac{1 - {p}}{{p}}} \! \! e^{B(2-\kappa)-\frac{(1-p)}{p}} e^{-B(1-\kappa) {n_0}^{p}} \enspace, \notag
	\end{align}
	where \eqref{eqn:cardUB} follows from \eqref{calculus bound},
	\eqref{eqn:MaxEst} holds due to \eqref{eq: bounding the max term},
	%	\eqref{eqn:MaxEst} holds since as
	%	%
	%	\[
	%	\max_{i \geq 0} e^{ - i B\kappa} (i + 1)^{\frac{1 - p}{p}}
	%	\! \leq \! \left[\frac{( 1- {p})}{B\kappa p}\right]^{\frac{1 - {p}}{{p}}} \! \! \! e^{\kappa[B-\frac{(1-p)}{\kappa p}]} \enspace,
	%	\]
	%	which in turn follows as $e^{ - x B\kappa} (x + 1)^{\frac{1 - p}{p}}$ has a global maximum at $x=\frac{(1-p)}{B\kappa p}-1.$
	%
	%Lastly, \eqref{eqn:RightRSum} follows by treating the sum as a right Riemann sum and since $\lfloor n_0^p \rfloor > n_0^p-1.$
	and \eqref{eqn:RightRSum} is obtained by treating the sum as a right Riemann sum and using $\lfloor n_0^p \rfloor > n_0^p-1.$
	This completes the proof of \eqref{eq: B formula}.

	We now prove \eqref{eq: pre-delta bound}.
	%	\gugan{First observe that the function $f(x) := (x + 1)^p \log[(x+ 1)/x],$ restricted to $x \geq 1,$ is strictly monotonically decreasing. This can be seen from the fact that $${x\log[(x+1)/x] \leq x \frac{1}{x} =1< 1/p}$$ for $x \geq 1,~p\in(0,1),$ which implies that $$f'(x)  = (x+1)^{p-1}\frac{p}{x}\left[x\log\left(\frac{x+1}{x}\right) - \frac{1}{p}\right]< 0$$ for $x \geq 1.$}
	Let $f(x) := (x + 1)^p \log[(x+ 1)/x].$ Notice that $ \lim_{x\rightarrow \infty} f(x) = 0$ for $x\geq 1,~p\in(0,1)$ because $f(x)$ is positive for $x>0$ and $$(x + 1)^p \log[(x+ 1)/x] \leq (x + 1)^p/x,$$ which goes to zero.	
	\gugan{Therefore, there is a $i_0 \equiv i_0(p, \hat{q}) \geq 0$  such that
		\begin{eqnarray*}
			(i+ 2)^p \log\left[\frac{i + 2}{i + 1}\right] \geq \frac{\hat{q}}{p} \enspace, & \text{ if $0 \leq i < i_0 \enspace,$}\\
			(i + 2)^p \log\left[\frac{i + 2}{i + 1}\right] \leq \frac{\hat{q}}{p} \enspace, & \text{ if $i \geq i_0$} \enspace.
		\end{eqnarray*}
		This is equivalent to saying that, for every $n \geq i + 2,$ if $0 \leq i < i_0,$ then
		\[
		(i + 1)^{-p} e^{-\hat{q} \sum_{k = i + 1}^{n - 1}(k + 1)^{-p}} \geq (i + 2)^{-p} e^{-\hat{q} \sum_{k = i + 2}^{n - 1}(k + 1)^{-p}}\enspace,
		\]
		and if $i_0 \leq i \leq n - 2,$ then
		\[
		(i + 1)^{-p} e^{-\hat{q} \sum_{k = i + 1}^{n - 1}(k + 1)^{-p}} \leq (i + 2)^{-p} e^{-\hat{q} \sum_{k = i + 2}^{n - 1}(k + 1)^{-p}}\enspace.
		\]
		Therefore, the maximum of $(i + 1)^{-p} e^{-\hat{q} \sum_{k = i + 1}^{n - 1} (k + 1)^{-p}}$ is achieved in one of the terminal values, i.e., at $i=0$ or $i=n-1.$ Thus,
		\begin{align}
		&\max_{0 \leq i \leq n - 1} (i + 1)^{-p} e^{-\hat{q} \sum_{k = i + 1}^{n - 1} (k + 1)^{-p}} \nonumber\\
		\leq&  \max\{e^{-\hat{q}\sum_{k = 1}^{n - 1} (k + 1)^{-p}}, n^{-p} \} \label{eq: Kg transition} \\
		\leq & K_g n^{-p} \enspace, \label{eqn:SupBdDer}
		\end{align}
		\sloppy
		where $K_g \geq 1$ (by its definition) is as defined in \eqref{Defn:Kg}.} The transition from \eqref{eq: Kg transition} to \eqref{eqn:SupBdDer} can be seen as follows. First, consider the case  $n\geq i_1,$ where $i_1$ is defined above \eqref{Defn:Kg}. In this case, by the definition of $i_1$, the maximum in \eqref{eq: Kg transition} is $n^{-p},$ which is bounded by $K_g n^{-p}.$  If ${n< i_1, \max\{n^{-p} \left( n^p e^{-\hat{q}\sum_{k = 1}^{n - 1} (k + 1)^{-p}}\right), n^{-p} \} \leq K_g n^{-p}}$ by the definition of $K_g.$
	
	\gugan{Now let $u_n := \sum_{k = 0}^{n - 1} [k+1]^{-p}.$ For $n \geq 1,$ we then have
		\begin{align}
		c_n
		= &
		\sum_{i = 0}^{n - 1} [i+1]^{-2p} e^{ -2 \hat{q} \sum_{k = i+1}^{n - 1}[k+1]^{-p}} \notag \\
		\leq & K_g n^{-p} \sum_{i = 0}^{n - 1} [i+1]^{-p} e^{ -\hat{q} \sum_{k = i+1}^{n - 1}[k+1]^{-p}} \label{eqn:supBd} \\
		= & K_g n^{-p}\sum_{i = 0}^{n - 1} [u_{i + 1} - u_i] e^{-\hat{q} [u_n - u_{i + 1}]} \label{eqn:DefnUnAppl}\\
		\leq & K_g e^{\hat{q}} n^{-p} e^{-\hat{q} u_n} \sum_{i = 0}^{n - 1} [u_{i + 1} - u_i] e^{\hat{q} u_{i}} \label{eqn:BeforeRiemannSum} \\
		\leq & K_g e^{\hat{q}} n^{-p} e^{-\hat{q} u_n}\int_{u_0}^{u_n}e^{\hat{q} s} \textnormal{d}s \label{eqn:LeftRiemannSum}\\
		= & K_g e^{\hat{q}} n^{-p} e^{-\hat{q} u_n} \frac{e^{\hat{q} u_n} - e^{\hat{q}u_0}}{\hat{q}} \notag\\
		\leq & \frac{K_g e^{\hat{q}}}{\hat{q}} n^{-p} \enspace, \label{eqn:expBd}
		\end{align}
		where \eqref{eqn:supBd} follows from \eqref{eqn:SupBdDer}, \eqref{eqn:DefnUnAppl} follows using the definition of $u_n,$ \eqref{eqn:BeforeRiemannSum} holds since $u_{i + 1} = u_i + (i + 1)^{-p} \leq u_i + 1$ for $i \geq 0,$ \eqref{eqn:LeftRiemannSum} follows by treating the sum above as a left Riemann sum, and, lastly, \eqref{eqn:expBd} holds as $u_0 = 0$ and $e^{-\hat{q} u_n} \leq 1.$}
	
	Consequently, for any $c>0$ and $n_0 \geq 1$,
	\[
	\sum_{n \geq n_0} \exp\!\left[\tfrac{-c \epsilon^2}{c_n}\right]
	\leq
	\sum_{n \geq n_0} \exp\!\left[{-\frac{c\hat{q} \epsilon^2}{K_g e^{\hat{q}}} n^p}\right] \enspace.
	\]
	The desired result now follows from \eqref{eq: B formula}.
	This completes the proof of the lemma.
\end{proof}

\section{ֿ\gal{Proof of Theorem~\ref{thm:condMain}}}
\label{sec:SupMat}
%This section contains all proofs of the lemmas and theorems presented in the paper, and provides additional technical results to support several of these proofs.
 \gal{As the analysis in Section~\ref{sec: analysis outline} is under assumption \eqref{eq: Gn' definition}, the results in the corresponding Subsections \ref{subsec:SuperSet}, \ref{appendix: perturbation bounds} and \ref{sec: conc_bound_appendx} here are under this assumption as well.}

\subsection{Application of \gugan{\vop\ } Formula in Subsection~\ref{subsec:Comparison}}
\label{sec:VoC}

%\galn{TODO: fix this}

Recall the definitions given below \eqref{eqn:tTotUpd}. On the interval $[\tI{k}, \tI{k + 1}),$ the functions $\zetT(\cdot)$ and $\zetM(\cdot)$ are constant, while $\zetD(\cdot)$ is linear. Therefore, the function $\zeta(t),$ $t \geq \tI{n_0},$ is piecewise continuous; specifically, it is continuous on the interval $[\tI{k}, \tI{k + 1}),$ for every $k \geq n_0.$ Separately, owing to the fact that it is a linear interpolation, the function $\bart(t),$ $t \geq \tI{n_0},$ is continuous everywhere.

The evolution described in \eqref{eqn:tTotUpd} can be viewed as a differential equation in integral form; further, it can be looked at as a perturbation of the ODE in \eqref{eqn:tLimODE}. It is then not difficult to see from \citep[Theorem 1.1.2]{lakshmikantham1998method} that \eqref{eqn:TrajComp-t} holds for any $t \in [\tI{n_0}, \tI{n_0 + 1}).$ Now, from the continuity of $\bart(t),$ it follows that  \eqref{eqn:TrajComp-t} holds even for $t = \tI{n_0 + 1},$ i.e.,
\begin{equation}
\label{eqn:limConseq}
\bart(\tI{n_0 + 1}) = \tSol{\tI{n_0 + 1}} + \Et(\tI{n_0 + 1}).
\end{equation}

Arguing in the same way as above, for any $t \in [\tI{n_0 + 1}, \tI{n_0+ 2}),$ it is easy to see that
\begin{equation}
\label{eqn:VoPTemp}
\bart(t) = \theta(t, \tI{n_0 + 1}, \bart(\tI{n_0 + 1})) + \int_{\tI{n_0 + 1}}^{t}e^{-\Xt(t - \tau)} \zeta(\tau) \df \tau \enspace.
\end{equation}
Moreover, observe that
\begin{eqnarray}
& & \theta(t, \tI{n_0 + 1}, \bart(\tI{n_0 + 1}))\\
& = & \theta\left(t, \tI{n_0 + 1}, \tSol{\tI{n_0 + 1}} + \Et(\tI{n_0 + 1})\right) \label{eqn:initPtExp}\\
& = & \thS + e^{-\Xt(t - \tI{n_0 + 1})} (\tSol{\tI{n_0 + 1}} + \Et(\tI{n_0 + 1}) - \thS) \label{eqn:unperODEuse1}\\
& = & \thS + e^{-\Xt(t - \tI{n_0 + 1})} (\tSol{\tI{n_0 + 1}} - \thS) +  e^{-\Xt(t - \tI{n_0 + 1})}\Et(\tI{n_0 + 1}) \notag \\
& = & \thS + e^{-\Xt(t - \tI{n_0 + 1})} (\tSol{\tI{n_0 + 1}} - \thS) +  \int_{\tI{n_0}}^{\tI{n_0 + 1}} e^{-\Xt(t - \tau)} \zeta(\tau) \df \tau \label{eqn:EtDefnUse}\\
& = & \theta(t, \tI{n_0 + 1}, \theta(\tI{n_0 + 1}, \tI{n_0}, \theta_{n_0})) +  \int_{\tI{n_0}}^{\tI{n_0 + 1}} e^{-\Xt(t - \tau)} \zeta(\tau) \df \tau \label{eqn:unperODEuse2}\\
&= & \theta(t, \tI{n_0}, \theta_{n_0}) +  \int_{\tI{n_0}}^{\tI{n_0 + 1}} e^{-\Xt(t - \tau)} \zeta(\tau) \df \tau \enspace \label{eqn:ExUniODE},
\end{eqnarray}
where \eqref{eqn:initPtExp} follows from \eqref{eqn:limConseq}; \eqref{eqn:unperODEuse1} and \eqref{eqn:unperODEuse2} hold as in \eqref{eqn:tOdeSol}; \eqref{eqn:EtDefnUse} follows from the definition of $\Et$ given below $\eqref{eqn:tTotUpd};$ and, finally, \eqref{eqn:ExUniODE} is true because of the uniqueness and existence of ODE solutions (see Picard-Lindel\"{o}f theorem).

Substituting \eqref{eqn:ExUniODE} in \eqref{eqn:VoPTemp}, it is easy to see that \eqref{eqn:TrajComp-t} holds for all $t \in [\tI{n_0 + 1}, \tI{n_0 + 2}).$ Inductively arguing this way, it follows that \eqref{eqn:TrajComp-t} holds for all $t \geq \tI{n_0}.$

\subsection{A \gal{Useful} Decomposition of the Event of Interest}
\label{subsec:SuperSet}

\gugan{
	For any event $\cE,$ let $\cE^c$ be its complement.  For all $n_0,~T > 0,$ define the event
	\begin{equation}
	\cE(n_0, T) := \{\norm{\bart(t) - \thS} \leq \et \; \forall t \geq \tI{n_0} + T + 1\} \cap \; \{\norm{\barz(s)} \leq \ez \; \forall s \geq \sI{n_0} + \xi(T) + 1\}\enspace, \label{eq: definition E(t)}
	\end{equation}
	where $\et, \ez$ are as in the statement of Theorem~\ref{thm:condMain}. Eventually, we shall use a bound on $\Pr\{\cE^c(n_0, T)\}$ to prove Theorem~\ref{thm:condMain}. Towards obtaining this bound, the aim here is to construct a well-structured superset for $\cE^c(n_0, T),$ assuming \eqref{eq: Gn' definition} holds, which is easier for analysis.
}

%
%For $T > 0,$ define the event
%%
%%
%

%For an event $\cE,$ let $\cE^c$ be its complement.
%Fix sufficiently large \gal{$n_0,~T > 0;$} we will say later how large. \gal{Pick $n_1 \equiv n_1(n_0)$ such that
{Fix some $T>0$ so that
	\begin{equation}
	\label{eqn:n1Cond}
	T \leq t_{n_1 + 1} - t_{n_0} = \sum_{k = n_0}^{n_1} \st_k \leq T + 1.
	\end{equation}
}%
%\noindent This is possible as \eqref{eqn:stepSizeBound} states that \gugan{$\sum_{n = 0}^{\infty} \st_n = \infty$ and $\sup_n \alpha_n \leq 1.$}

%Our aim here is to construct a superset for the event \gal{ $ \cE^c(n_0, T)$  (defined in \eqref{eq: definition E(t)}), assuming \eqref{eq: Gn' definition} holds,} which is easier for analysis. The superset additionally contains the information of what happens \gal{between times $n_0$ and $n_1.$}

%
%\begin{remark}
%%
%\label{rem:MonDec}
%By standard ODE literature,  $\lim_{t \to \infty} \tSol{t} = \thS.$	As $\Xt$ is positive definite by  \ref{assum:posDef}, $\frac{d}{dt}\|\theta(t) - \thS\|^2 = -2(\theta(t) - \thS)^\top \Xt (\theta(t) - \thS)<0$; hence, assuming \eqref{eq: Gn' definition} holds,
%%
%\[
%\norm{\tSol{t} - \thS} \leq \Rti \quad \forall t \geq \tI{n_0}.
%\]
%%
%Likewise, $\lim_{s \to \infty} \zSol{s} = \zS$  and
%%
%\[
%\norm{\zSol{s}} \leq \Rzi \quad \forall s \geq \sI{n_0}.
%\]
%%
%\end{remark}

By Remark~\ref{rem:MonDec}, $\tSol{t}$ stays in the $\Rti-$radius ball  around $\thS$ for all \gal{$t \geq \tI{n_0},$} and $\zSol{s}$ stays in the $\Rzi-$radius ball around $\zS$ for all \gal{$s \geq \sI{n_0} .$} But the same cannot be said for $\bart(t)$ and $\barz(s)$ due to the presence of noise. We  show instead that these lie with high probability in bigger but fixed radii balls $\Rto$ and $\Rzo,$ where $\Rzo > \Rzi$ is an arbitrary constant, and
%
%{TODO: Adjust $\ez$ and $\et$ to $\etg$ and $\ezg$, not the other way around.}
%Fix $\Rto, \Rzo > 0$ such that $\Rto > \Rti,$ $\Rzo > \Rzi,$
{
	\begin{equation}
	\label{eq: defn of Rto}
	\Rto := \Rti + \frac{4\Kt \norm{\Wt} \Kz \Rzi}{(\lmin-\lm)e}
	\enspace.
	\end{equation}
	Note that, by the choice of $\et$ and $\ez$}
\begin{equation} \label{eq:epsilon smaller than R}
\etg := \Rto - \Rti \geq \et \enspace, \text{ and }\ezg  := \Rzo - \Rzi \geq \ez \enspace.
\end{equation}
%
%We shall describe later how large $\Rto$ and $\Rzo$ should be.
%
For \gal{$n \geq n_0,$} let
\begin{equation}
\label{defn:rhot}
\rt_{n + 1} := \sup\limits_{\tau \in [\tI{n}, \tI{n + 1}]}\norm{\bart(\tau) - \tSol{\tau}} \enspace, \quad
\rtS_{n + 1} := \sup\limits_{\tau \in [\tI{n}, \tI{n + 1}]} \norm{\bart(\tau) - \thS} \enspace,
\end{equation}

\begin{equation}
\label{defn:rhoz}
\rz_{n + 1} :=  \sup\limits_{\mu \in [\sI{n}, \sI{n + 1}]} \norm{\barz(\mu) - \zSol{\mu}}\enspace, \quad
\rzS_{n + 1} := \sup\limits_{\mu \in [\sI{n}, \sI{n + 1}]} \norm{\barz(\mu)}  \enspace,
\end{equation}
and define the (``good'') event
\begin{equation}
\label{defn:Gn}
G_n := \{ \norm{\bart(\tau) - \thS} \leq \Rto \; \forall \tau \in [\gal{\tI{n_0}}, \tI{n}]\}  \cap \{\norm{\barz(\mu)} \leq \Rzo \; \forall \mu \in [ \gal{\sI{n_0}}, \sI{n}]\}.
\end{equation}
Additionally, define the (``bad'') events $\cE_{\aftE} := \bigcup_{n > n_1}\left[ \{\rtS_{n + 1} > \et\} \cup \{\rzS_{n + 1} > \ez\}\right]$ and
\begin{equation*}
	\cE_{\midE}  := \left\{\Big[\sup_{\gal{n_0} \leq n \leq n_1} \rt_{n + 1}\Big] \geq \etg \right\}
	\cup \left\{\Big[\sup_{\gal{n_0} \leq n \leq n_1} \rz_{n + 1}\Big] \geq \ezg \right\} \enspace.
\end{equation*}

\gal{The} desired superset stated at the beginning of this subsection is given below.

\begin{lemma}[Decomposition of Event of Interest] \label{lem:IntEv}
	\gal{Consider \eqref{eq: definition E(t)} and suppose that \eqref{eq: Gn' definition} holds. Then}
	\gal{\begin{align*}
			\cE^c(n_0, T) \subseteq
			\bigcup_{n = n_0}^{n_1} \{ \! G_n \cap \left[ \{\rt_{n + 1} \geq \etg\} \! \cup \! \{\rz_{n + 1} \geq \ezg\}\right]\} \\ \cup
			\bigcup_{n > n_1} \left[ G_n \cap  \left[ \{\rtS_{n + 1} \geq \et\} \cup \{\rzS_{n + 1} \geq \ez\}\right]\right].
		\end{align*}
	}
\end{lemma}

\begin{proof}
By \eqref{eqn:n1Cond}, as $\tI{n_1 + 1} \leq   T + 1,$ $\cE^c(T) \subseteq \cE_{\aftE}.$ For any two events $\cE_1$ and $\cE_2,$ as
\[
\cE_1 = [\cE_2 \cap \cE_1] \cup [\cE^c_2 \cap \cE_1] \subseteq \cE_2 \cup [\cE_2^c \cap \cE_1],
\]
we have $\cE_{\aftE} \subseteq \cE_{\midE} \cup [\cE^c_{\midE} \cap \cE_{\aftE}].$ Using Remark~\ref{rem:MonDec} \gal{and since \eqref{eq: Gn' definition} holds},
\[
\left\{\Big[\sup_{n_0 \leq k < n}  \rt_{k + 1} \Big] \leq \etg \right\} \cap \left\{\Big[\sup_{n_0 \leq k < n}  \rz_{k + 1} \Big] \leq \ezg \right\} \subseteq  G_n.
\]
for all $n \geq n_0.$ Hence by simple manipulations, we have
\[
\cE_{\midE} \subseteq
\bigcup_{n = n_0}^{n_1} \{ G_n \cap \left[ \{\rt_{n + 1} \geq \etg\} \cup \{\rz_{n + 1} \geq \ezg\}\right]\}.
\]
Arguing similarly, one can see that
\begin{eqnarray*}
& & \cE^c_{\midE} \cap \cE_{\aftE}\\
& \subseteq & G_{n_1 + 1} \cap \cE_{\aftE}\\
& \subseteq & \hspace{-0.5em} \bigcup_{n > n_1} \left[ G_n \cap \left[ \{\rtS_{n + 1} \geq \et\} \cup \{\rzS_{n + 1} \geq \ez\}\right]\right],
\end{eqnarray*}
where the last inequality follows as $\et \leq \Rto$ and $\ez \leq \Rzo.$ The desired result now follows.
\end{proof}

\subsection{\gal{Technical Lemmas for Subsection~\ref{appendix: perturbation bounds}}}

\gal{We now provide two technical lemmas that will be used in the proofs of Lemmas~\ref{lem:EzDBd} and \ref{lem:EtDBd}.}

\begin{lemma}
	\label{lem: technical for norm of int}
	Let $0< r_0 < r_1 < \cdots < r_\ell$, let $\gamma_i = r_{i+1} - r_i$ for $i=0, \dots, \ell-1$, let $U$ be some $d \times d$ matrix, and let $\rho: \Real \to \Real$ be some mapping.
	%Let, furthermore, $U$ be some $d \times d$ matrix, and let $E(r_\ell,r_0) = \int_{r_0}^{r_\ell} e^{U()}$.
	Assume that for some constant $J$ it holds that $\|\rho(\sigma)\| \leq \gamma_i J$ for any $\sigma \in [r_i,r_{i+1}]$ and $i=0, \dots, \ell-1$.
	Assume, furthermore that for some constants $K>0$ and $q_0>0$ it holds that $\|e^{-U(r-r_0)}\| \leq Ke^{-q_0(r-r_0)}$ for any $r>r_0$.
	Then
	\[
	\left\|\int_{r_0}^{r_\ell} e^{-U(r_\ell-\sigma)}\rho(\sigma)d\sigma\right\|
	\leq
	\frac{KJ}{q_0}\left[\sup_{i=0, \dots, \ell-1} \gamma_i\right] .
	\]
\end{lemma}
\begin{proof}
	The claim of the lemma follows easily as, due to the assumptions,
	\begin{align*}
		\left\|\int_{r_0}^{r_\ell} e^{-U(r_\ell-\sigma)}\rho(\sigma)d\sigma\right\| &\leq
		\sum_{i=0}^{\ell-1} \int_{r_i}^{r_{i+1}} \left\|e^{-U(r_\ell-\sigma)}\right\|\left\|\rho(\sigma)\right\| \df \sigma
		\\
		&\leq
		KJ\sum_{i=0}^{\ell-1} \gugan{\gamma_i} \int_{r_i}^{r_{i+1}} e^{-q_0(r_\ell-\sigma)} \df \sigma
		\\
		&\leq
		KJ \left[\sup_{i=0, \dots, \ell-1} \gamma_i\right] \gugan{ \int_{r_0}^{r_\ell} e^{-q_0(r_\ell - \sigma)}\df  \sigma}
		\\
		&\leq
		\frac{KJ}{q_0} \left[\sup_{i=0, \dots, \ell-1} \gamma_i\right],
	\end{align*}
	\gugan{where, to get the last relation, we have used the fact that $\int_{r_0}^{r_\ell}e^{-q_0(r_\ell - \sigma)} \df \sigma \leq 1.$}
\end{proof}

\begin{lemma}[Dominating Decay Rate Bound]
	\label{lem:comDRate}
	Fix $\lm \in (0, \lmin)$ where $ \lmin := \min\{\lt, \lz\}.$ Then for $n \geq n_0,$
	\[
	\sum_{k = n_0}^{n - 1} \int_{\tI{k}}^{\tI{k + 1}}e^{-\lt(\tI{n} - \tau)} e^{-\lz(\xi(\tau) - \sI{n_0})} \df \tau  \leq  \frac{1}{(\lmin-\lm)e} e^{- q(\tI{n} - \tI{n_0})}.
	\]
\end{lemma}

\begin{proof}
	From \eqref{eqn:stepSizeBound}, $\beta_k \geq \alpha_k$ $\forall k \geq 1.$ Using this and  \eqref{defn:xi}, $\forall k \geq 1$ and $\tau \in [\tI{k}, \tI{k + 1}],$ $\xi(\tau) - \sI{k} \geq \tau - \tI{k}.$ Hence for any $\tau \in [\tI{n_0}, \tI{n}],$
	\[
	-\lt(\tI{n} - \tau) - \lz(\xi(\tau) - \sI{n_0}) \leq -\lmin(\tI{n} - \tI{n_0}).
	\]

	Now, since $\frac{1}{\alpha e}$ is the maximum of $xe^{-\alpha x}$,
	\begin{eqnarray*}
		(\tI{n} - \tI{n_0})e^{-\lmin(\tI{n} - \tI{n_0})}
		& = &(\tI{n} - \tI{n_0}) e^{-(\lmin-\lm)(\tI{n} - \tI{n_0})} e^{-\lm(\tI{n} - \tI{n_0})}\\
		& \leq &  \frac{1}{(\lmin-\lm)e}e^{-\lm(\tI{n} - \tI{n_0})}.
	\end{eqnarray*}
	The desired result now follows.
\end{proof}

\subsection{\gal{Bounding the Error Terms Discussed in Subsection~\ref{subsec:conc_bound_lemmas}}}
\label{appendix: perturbation bounds}
For obtaining the bounds in this subsection, we first show worst-case bounds on the increments.
For $k \geq n_0,$ let
\begin{equation}
\label{defn:Itk}
\It(k) := \norm{\theta_{k + 1} - \theta_k} / \st_k
\end{equation}
and
\begin{equation}
\label{defn:Izk}
\Iz(k) := \norm{z_{k + 1} - z_k} / \sw_k.
\end{equation}
Also, \gal{let}
\begin{equation}
\label{eq:R-star defn}
\Rs := \norm{\Xt^{-1}} \norm{\bt}
\end{equation}
so that
\begin{equation}
\label{eqn:thsBd}
\norm{\thS} \leq  \Rs.
\end{equation}
On $G_n,$ for $k \in \{n_0, \ldots, n\},$
\begin{align}
	\norm{w_k}   & \leq  \norm{z_k} +  \norm{\lambda(\thS)} + \norm{\lambda(\theta_k) - \lambda(\thS)} \nonumber
	\\
	& \leq  \Rzo + \norm{\Ww^{-1}} \big[\norm{\vw} + \norm{\Tw}[\Rs + \Rto] \big] \nonumber \\
	& =:  \Rwo. \label{eqn:wBd}
\end{align}

\begin{lemma}[Bounded Differences]
\label{lem:ItkBd}
Fix $n_0 \geq 0$ and $n \geq n_0.$ Then on $G_n,$ \gugan{assuming \eqref{eq: Gn' definition}},
\begin{align}
\sup_{n_0 \leq k \leq n} \It(k) \leq \Jt,\quad \sup_{n_0 \leq k \leq n} \Iz(k) \leq \Jz.
\end{align}
where
\[
\Jt = \norm{\vt} + \norm{\Tt}[ \Rs  + \Rto] + \norm{\Wt} \Rw + \mt [1 + \Rs + \Rto + \Rw]
\]
and
\[
\Jz := \norm{\Ww} \Rzo  +  \norm{\Ww^{-1}} \norm{\Tw} \Jt + \mw(1 + \Rs + \Rto + \Rw).
\]
\end{lemma}
\begin{proof}
Fix $k \in \{n_0, \ldots, n\}.$ On $G_n,$ using \eqref{eq:theta_iter}, \ref{assum:Noise}, \eqref{eqn:thsBd}, \eqref{defn:Gn}, and \eqref{eqn:wBd}, in that order,
\begin{eqnarray}
\It(k) & \leq & \norm{\vt - \Tt \theta_k - \Wt w_k} + \norm{\Mt_{k + 1}} \nonumber \\
& \leq & \norm{\vt} + \norm{\Tt}( \norm{\thS} + \norm{\theta_k - \thS}) \nonumber \\
& & + \norm{\Wt} \norm{w_k}  \nonumber \\
& & + \mt [1 + \norm{\thS} + \norm{\theta_k - \thS} + \norm{w_k}] \nonumber\\
& \leq & \Jt. \label{eq:Jt}
\end{eqnarray}

Similarly, on $G_n,$ using \eqref{eq:z_iter}, \ref{assum:Noise}, \eqref{defn:Gn}, \eqref{eqn:stepSizeBound} from \ref{assum:stepSize} %the fact that $\eta_k \leq 1,$
, \eqref{eq:Jt}, \eqref{eqn:thsBd}, and \eqref{eqn:wBd}, in that order,
\begin{eqnarray*}
\Iz(k) & \leq & \norm{\Ww} \norm{z_k} + \norm{\lambda(\theta_k) - \lambda(\theta_{k + 1})} /\sw_k \\
& &  + \norm{\Mw_{k + 1}}\\
& \leq & \norm{\Ww} \norm{z_k} +  \norm{[\Ww]^{-1}} \norm{\Tw} \eta_k \It(k)\\
& & + \mw(1 + \norm{\thS} + \norm{\theta_k - \thS} + \norm{w_k})\\
& \leq & \Jz.
\end{eqnarray*}
Since $k$ was arbitrary the result follows.
\end{proof}

Let $q^{(1)}(\Ww), \ldots, q^{(d)}(\Ww)$ be the eigenvalues of $\Ww.$ Fix $\lz \in (0, \lzp),$ where $$\lzp := \min_i \{\rl(q^{(i)}(\Ww))\}.$$ Then from Corollary 3.6 \citep{teschl2004ordinary}, there exists $\Kz \geq 1$ so that
\begin{equation}
\label{eqn:zMatrixBd}
\norm{e^{-\Ww (s - \mu)}} \leq \Kz e^{-\lz(s - \mu)}, \; \forall s \geq \mu.
\end{equation}

For the rest of the results in this subsection, we consider intermediate intervals $[s_n,s_{n+1}]$. The next lemma gives bounds on the three error terms of the interpolated trajectory $\barz(s)$  at the extremes $\{s_n,s_{n+1}\}$. \gal{This suffices for bounding the deviation of $\barz(s)$ from $z(s)$ on the whole interval, as is shown in the subsequent lemma}.

\begin{lemma}[Perturbation Error Bounds for $z_n$]
\label{lem:EzDBd}
Fix $n_0 \geq 0$ and $n \geq n_0$ Then on $G_n,$ \gugan{assuming \eqref{eq: Gn' definition}},
\begin{align*}
&\sup_{\ell \in \{n, n + 1\}} \norm{\EzD(\sI{\ell})} \leq \LzD \left[\sup_{k \geq 0} \sw_k\right], \\
&\sup_{\ell \in \{n, n+ 1\}} \norm{\EzS(\sI{\ell})} \leq \LzS \left[\sup_{k \geq 0} \eta_k \right],\\
&\norm{\EzM(\sI{n + 1})} \leq \Kz \norm{\EzM(\sI{n})} +  \LzM \sw_n. 	
\end{align*}
where
%	\begin{align*}
$\LzD := \frac{\Kz \Jz \norm{\Ww}}{ \lz}$,
$\LzS := \frac{\Kz \norm{\Ww^{-1}}\norm{\Tw} \Jt}{\lz}$, $\LzM := \Kz \mw[1 + \Rs + \Rto + \Rw].$
%	\end{align*}
\end{lemma}

\begin{proof}
Fix $\ell \in \{n, n + 1\}.$

For the first claim
note that, by Lemma~\ref{lem:ItkBd}, on $G_n,$
\begin{align*}
\norm{\chizD(\mu)} \leq \norm{\Ww} (\mu - \sI{k}) \Iz(k) \leq \norm{\Ww} \beta_k \Jz %\Iz(k) ,
\end{align*}
for $\mu \in [\sI{k}, \sI{k + 1})$, where $\Iz(k)$ is as in \eqref{defn:Izk}.
The claim then follows easily by recalling \eqref{eqn:zMatrixBd}, and applying Lemma~\ref{lem: technical for norm of int}
% \gugan{(To do: Mention this Lemma before)}
 with $r_i = \sI{i}$, $\gamma_i=\beta_i$, $U = \Ww$, $\rho = \chizD$, $K = \Kz$, $q_0 = -\lz$ and $J=\norm{\Ww}\Jz$.

%By triangle inequality,
%
%\[
%\norm{\EzD(\sI{\ell})} \leq \sum_{k = 0}^{\ell - 1}
%\int_{\sI{k}}^{\sI{k + 1}} \norm{e^{-\Ww(\sI{\ell} - \mu)}} \norm{\chizD(\mu)} \df \mu.
%\]
%
%Combining this with \eqref{eqn:zMatrixBd}, we have
%
%\[
%\norm{\EzD(\sI{\ell})} \leq \Kz \sum_{k = 0}^{\ell - 1}
%\int_{\sI{k}}^{\sI{k + 1}} e^{-\lz(\sI{\ell} - \mu)} \norm{\chizD(\mu)} \df \mu.
%\]
%
%Let $k \in \{0, \ldots, \ell - 1\}$ and $\mu \in [\sI{k}, \sI{k + 1}).$ Then
%
%\[
%\norm{\chizD(\mu)} \leq \norm{\Ww} (\mu - \sI{k}) \Iz(k) \leq \norm{\Ww} \beta_k \Iz(k),
%\]
%
%where $\Iz(k)$ is as in \eqref{defn:Izk}. Lemma~\ref{lem:ItkBd}, the above two relations and the fact that
%$\sum_{k = 0}^{\ell - 1}\int_{\sI{k}}^{\sI{k + 1}} e^{-\lz(\sI{\ell} - \mu)} \df \mu \leq 1/\lz$
%now imply that, on $G_n,$
%
%\begin{equation} \label{eq:first claim}
%\norm{\EzD(\sI{\ell})} \leq \LzD \left[\sup_{k \geq 0} \sw_k\right].
%\end{equation}
%
%This proves the first claim.

For the second claim, let $k \in \{n_0, \ldots,\ell - 1\}$ and $\mu \in [\sI{k}, \sI{k + 1}).$ With $\It(k)$  as in \eqref{defn:Itk},
\[
\norm{\chizS(\mu)} \leq \eta_k \norm{\Ww^{-1}} \norm{\Tw} \It(k).
\]
Hence by Lemma~\ref{lem:ItkBd}, on $G_n,$
\[
\norm{\chizS(\mu)} \leq \eta_k \norm{\Ww^{-1}} \norm{\Tw} \Jt.
\]
The claim then follows again by \eqref{eqn:zMatrixBd} and Lemma~\ref{lem: technical for norm of int}.
%Arguing as for \eqref{eq:first claim}, the desired result follows easily.

For the third claim, by its definition and the triangle inequality,
\begin{eqnarray*}
& & \norm{\EzM(\sI{n + 1})}\\
& = &  \norm{\int_{s_{n_0} }^{\sI{n + 1}} e^{-\Ww (\sI{n + 1} - \mu)} \chizM(\mu)\df \mu}\\
& \leq & \norm{e^{-\Ww \beta_n} \int_{ s_{n_0}}^{\sI{n}} e^{-\Ww (\sI{n} - \mu)} \chizM(\mu)\df \mu} + \norm{\int_{\sI{n}}^{\sI{n + 1}} e^{-\Ww (\sI{n + 1} - \mu)} \chizM(\mu)\df \mu}.
\end{eqnarray*}
Applying \eqref{eqn:zMatrixBd} on both terms, we get that
\[
\norm{\EzM(\sI{n + 1})} \leq \Kz \norm{\EzM(\sI{n})} + \Kz \sw_n \norm{\Mw_{n + 1}}.
\]
On $G_n,$ using \ref{assum:Noise} with \eqref{defn:Gn}, \eqref{eqn:thsBd}, and \eqref{eqn:wBd}, we have $\Kz\norm{\Mw_{n + 1}} \leq \LzM.$ The third claim is now easy to see.
\end{proof}

The \gal{next} lemma shows that for $\tau \in [s_n,s_{n+1}]$, $\barz(\tau)$  cannot deviate much from the ODE trajectory $z(\tau)$ if the stepsizes are small enough. In particular, it bounds the distance with decaying terms using Lemma~\ref{lem:EzDBd}.

\begin{lemma}[ODE-SA Distance Bound for $z_n$]
\label{lem:rzBd}
Fix $n_0 \geq 0$ and $n \geq n_0.$ Then on $G_n$ \gal{and since \eqref{eq: Gn' definition} holds},
\begin{align*}
\rz_{n + 1} \leq &\Kz \norm{\EzM(\sI{n})} + \Lz \max \left\{\sup_{k \geq n_0} \beta_k, \sup_{k \geq n_0} \eta_k \right\}, \\
\rzS_{n + 1}  \leq & \Kz \norm{\EzM(\sI{n})} + \Kz \Rzi e^{-\lz(\sI{n} - \sI{n_0})} + \Lz \max \left\{\sup_{k \geq n_0} \beta_k, \sup_{k \geq n_0} \eta_k \right\},
\end{align*}
where $\Lz = \LzD + \LzM + \norm{\Ww} \Rzi +  \LzS.$
\end{lemma}

\begin{proof}
Let $\mu \in [\sI{n}, \sI{n+ 1}].$ Then there exists $\kappa \in [0,1]$ so that
\[
\barz(\mu) = (1 - \kappa) \barz(\sI{n}) + \kappa \barz(\sI{n + 1}).
\]
Hence
\[
\norm{\barz(\mu) - \zSol{\mu}} \leq   (1 - \kappa) \norm{\barz(\sI{n}) - \zSol{\mu}} + \kappa \norm{\barz(\sI{n + 1}) - \zSol{\mu}}.
\]
Using \eqref{eqn:zLimODE},
\[
\zSol{\mu} = \zSol{\sI{n}} + \int_{\sI{n}}^{\mu} [- \Ww \; \zSol{\mu_1}] \df \mu_1,
\]
and
\[
\zSol{\sI{n + 1}} = \zSol{\mu} + \int_{\mu}^{\sI{n + 1}} [- \Ww \; \zSol{\mu_1}] \df \mu_1.
\]
Combining the above three relations, we have
\begin{multline*}
\|\barz(\mu) - \zSol{\mu}\| \leq (1 - \kappa) \norm{\barz(\sI{n}) - \zSol{\sI{n}}}\\ + \kappa \norm{\barz(\sI{n + 1}) - \zSol{\sI{n + 1}}} + \int_{\sI{n}}^{\sI{n + 1}} \norm{\Ww} \norm{\zSol{\mu_1}} \df \mu_1.
\end{multline*}
\gal{Since \eqref{eq: Gn' definition} holds, as} $\norm{z_{n_0}} \leq \Rzi,$ from Remark~\ref{rem:MonDec}, $\norm{\zSol{\mu}} \leq \Rzi$ for all $s \geq \sI{n_0}.$ Using this with \eqref{eqn:w-TrajComp}, \eqref{eqn:zMatrixBd}, the facts that $\Kz \geq 1$ and $\beta_n \leq [\sup_{k \geq n_0} \beta_k],$ and Lemma~\ref{lem:EzDBd}, the first claim follows:
\begin{align}
\rz_{n + 1}
\leq
&\LzD\left[\sup_{k \geq n_0}\beta_k\right] +  \LzS\left[\sup_{k \geq n_0}\eta_k\right]  + \kappa\LzM\beta_n
\notag
\\
&+ ((1-\kappa)+ \kappa \Kz)\norm{\EzM(\sI{n})} +\norm{\Ww} \beta_n \Rzi.
\notag
\\
\leq& \Kz \norm{\EzM(\sI{n})} + \Lz \max \left\{\sup_{k \geq n_0} \beta_k, \sup_{k \geq n_0} \eta_k \right\}.
\label{eq: first claim}
\end{align}

For the second claim observe that
\[
\norm{\barz(\mu)} \leq \norm{\barz(\mu) - \zSol{\mu}} + \norm{\zSol{\mu}}.
\]
Hence
\[
\rzS_{n + 1} \leq \rz_{n +1} + \sup_{\mu \in [\sI{n}, \sI{n + 1}]} \norm{\zSol{\mu}}.
\]
Lastly, \gal{since \eqref{eq: Gn' definition} holds, }$\norm{z_{n_0}} \leq \Rzi,$ and hence using \eqref{eqn:zOdeSol} and \eqref{eqn:zMatrixBd},
\[
\norm{\zSol{\mu}} \leq \Kz \Rzi e^{- \lz (\mu - \sI{n_0})}.
\]
Combining the above two relations with \eqref{eq: first claim}, the desired result is now easy to see.
\end{proof}

We now reproduce the results of Lemma~\ref{lem:rzBd},  this time for $\{\theta_n\}$ instead of $\{z_n\}$, and obtain bounds on $\rt_{n + 1}$ and $\rtS_{n + 1}$ on $G_n$, \gal{assuming \eqref{eq: Gn' definition}}.
%
%For $k \geq n_0$ and $\tau \in [\tI{k}, \tI{k + 1}),$ let
%%
%\begin{align*}
%\zetD(\tau) & := \hth(\theta_k, \lambda(\theta_k))  - \hth(\bart(\tau), \lambda(\bart(\tau)))\\
%&  = \Xt[\bart(\tau) - \theta_k],\\
%\zetM(\tau) & := \Mt_{k + 1},\\
%\zetT(\tau) & := \hth(\theta_k, w_k) - \hth(\theta_k, \lambda(\theta_k)) = - \Wt z_k.
%\end{align*}
%Using simple manipulations on \eqref{eq:theta_iter}, for any $t \geq \tI{n_0},$
%%
%\[
%\bart(t) = \bart(\tI{n_0})  + \int_{\tI{n_0}}^{t}\left[\hth\Big(\bart(\tau), \lambda(\bart(\tau))\Big) + \zeta(\tau)\right] \df \tau,
%\]
%%
%where $\zeta(\tau) = \zetD(\tau) + \zetM(\tau) + \zetT(\tau).$ These are respectively perturbations due to discretization, martingale difference noise, and error in tracking the equilibrium of \eqref{eqn:wLimODE}. Recall that as $\theta_n$ evolves, the equilibria of \eqref{eqn:wLimODE} moves. The tracking error is a function of the $z_n$ which, from \eqref{eq:z_iter}, is the difference between $w_n$ and $\lambda(\theta_n).$  By the \vop\ formula,
%%
%\begin{equation}
%\label{eqn:tAlekseev}
%\bart(t) = \tSol{t} + \Et(t),
%\end{equation}
%%
%where $\Et(t) = \EtD(t) + \EtM(t) + \EtT(t)$ with
%%
%\[
%\EtD(t) = \int_{\tI{n_0}}^{t} e^{- \Xt(t - \tau)} \zetD(\tau) \df \tau,
%\]
%%
%and similarly for $\EtM(t)$ and $\EtT(t)$. As in Subsection~\ref{subsec:zBd},
\gal{To do so, it suffices to bound $\norm{\EtD(\cdot)},$ $\norm{\EtM(\cdot)},$ and $\norm{\EtT(\cdot)}$ on the interval $[\tI{n}, \tI{n + 1}].$}

Similarly as in \eqref{eqn:zMatrixBd}, there exist $\lt$ and $\Kt \geq 1$ so that
\begin{equation}
\label{eqn:tMatrixBd}
\norm{e^{-\Xt(t - \tau)}} \leq \Kt e^{-\lt(t - \tau)}, \; \forall t \geq \tau.
\end{equation}

Fix
\begin{equation}
\label{eq:lmdef}
\lm \in (0, \lmin), \hspace{1cm} \lmin := \min\{\lt, \lz\},
\end{equation}
where $\lz$  is from \eqref{eqn:zMatrixBd}.  The next lemma gives bounds on the three components of $\Et(t)$.

\begin{lemma}[Perturbation Error Bounds for $\theta_n$]
\label{lem:EtDBd}
Fix $n_0 \geq 0$ and $n \geq n_0.$ Then on $ G_n,$ \gal{assuming \eqref{eq: Gn' definition}},
\begin{align*}
\sup_{\ell \in \{n, n + 1\}} \norm{\EtD(\tI{\ell})} & \leq \LtD \left[\sup_{k \geq n_0} \st_k\right], \\
\sup_{\ell \in \{n, n + 1\}} \norm{\EtT(\tI{\ell})} &\leq \LtT{a} \; e^{-\lm(\tI{n} - \tI{n_0})} + \LtT{b}\left[\sup_{k \geq n_0}  \beta_k\right]
+  \LtT{c} \left[\sup_{n_0 \leq k \leq n} \rz_{k + 1}\right], \\
\norm{\EtM(\tI{n + 1})} & \leq \Kt \norm{\EtM(\tI{n})} +  \LtM \st_n,
\end{align*}
where  $\LtD := \frac{\Kt \Jt \norm{\Xt}}{\lt},$
$\LtT{a}  :=  \Kt \norm{\Wt} \Kz \Rzi  \frac{1}{(\lmin-\lm)e},$ $\LtT{b}  :=  \Kt \norm{\Wt}\norm{\Ww} \Rzi/ \lt,$ $\LtT{c}  :=  \Kt \norm{\Wt} /\lt,$ $\LtM := \Kt \mt[1 + \Rs + \Rto + \Rw].$

\end{lemma}

\begin{proof}
For the first claim of the lemma fix $\ell \in \{n, n + 1\}.$ Let $k \in \{n_0, \ldots,\ell - 1\}$ and  $\tau \in [\tI{k}, \tI{k + 1}).$ With $\It(k)$ as in \eqref{defn:Itk},
\[
\norm{\zetD(\tau)} \leq \norm{\Xt}  (\tau - \tI{k}) \It(k) \leq \alpha_k \norm{\Xt} \It(k).
\]
So by Lemma~\ref{lem:ItkBd}, on $G_n,$ $\norm{\zetD(\tau)} \leq \alpha_k \norm{\Xt} \Jt.$ The first claim now follows
by \eqref{eqn:tMatrixBd} and Lemma~\ref{lem: technical for norm of int}.
% as in the proof of Lemma~\ref{lem:EzDBd}.

For proving the second claim of the lemma  let $\ell = n.$ By triangle inequality,
\[
\|\EtT(\tI{n})\| \leq \sum_{k = n_0}^{n - 1} \int_{\tI{k}}^{\tI{k + 1}} \norm{e^{-\Xt(\tI{n} - \tau)}} \norm{\zetT(\tau)} \df \tau.
\]
Using \eqref{eqn:tMatrixBd}, it follows that
\[
\|\EtT(\tI{n})\| \leq  \Kt \sum_{k = n_0}^{n - 1} \int_{\tI{k}}^{\tI{k + 1}} e^{-\lt(\tI{n} - \tau)} \norm{\zetT(\tau)} \df \tau.
\]
Fix $k \in \{n_0, \ldots, n - 1\}$ and $\tau \in [\tI{k}, \tI{k + 1}).$ Then
\[
\norm{\zetT(\tau)} \leq \norm{W_1} \norm{z_k}.
\]
Using \eqref{defn:xi} and the triangle inequality,
\begin{multline*}
\norm{z_k} \leq \norm{\zSol{\xi(\tau)}}\\  + \norm{\zSol{\xi(\tau)} - \zSol{\xi(\tI{k})}}
+  \norm{z_k - \zSol{\xi(\tI{k})}}.
\end{multline*}
\gal{Since \eqref{eq: Gn' definition} holds}, $\norm{z_{n_0}} \leq \Rzi;$ thus by \eqref{eqn:zOdeSol} and  \eqref{eqn:zMatrixBd},
\[
\norm{\zSol{\xi(\tau)}} \leq \Kz \Rzi e^{-\lz(\xi(\tau) - \sI{n_0})}.
\]
Remark~\ref{rem:MonDec} also implies that, as $\norm{z_{n_0}} \leq \Rzi,$ $\norm{\zSol{s}} \leq \Rzi$ for all $s \geq \sI{n_0}.$ Hence by \eqref{eqn:zLimODE},
\begin{eqnarray*}
\norm{\zSol{\xi(\tau)} - \zSol{\xi(\tI{k})}} & \leq & \norm{ \int_{\xi(\tI{k})}^{\xi(\tau)}[-\Ww]\; \zSol{\mu} \df \mu}\\
& \leq & \norm{\Ww}\Rzi \beta_k,
\end{eqnarray*}
where the last relation holds as $[\xi(\tau) - \xi(\tI{k})] \leq [\sI{k + 1} - \sI{k}].$ Also note that, by \eqref{defn:rhoz},
\[
\norm{z_k - \zSol{\xi(\tI{k})}} \leq \rz_{k + 1}.
\]
Combining the above relations,
\begin{eqnarray*}
& & \norm{\zetT(\tau)}\\
& \leq & \norm{\Wt}\bigg[\Kz \Rzi e^{-\lz(\xi(\tau) - \sI{n_0})} + \norm{\Ww} \Rzi \beta_k + \rz_{k + 1}\bigg]\\
& \leq & \norm{\Wt}\Bigg[\Kz \Rzi e^{-\lz(\xi(\tau) - \sI{n_0})} + \norm{\Ww} \Rzi \left[\sup_{k \geq n_0} \beta_k\right] + \left[\sup_{n_0 \leq k \leq n - 1}\rz_{k + 1}\right]\Bigg]
\end{eqnarray*}
By Lemma~\ref{lem:comDRate} and the fact that $\int_{\tI{n_0}}^{\tI{n}} e^{-\lt(\tI{n} - \tau)} \df \tau \leq 1/\lt,$
\[
\|\EtT(\tI{n})\| \leq  \LtT{a} e^{-\lm(\tI{n} - \tI{n_0})} + \LtT{b}  \left[\sup_{k \geq n_0} \beta_k\right] + \LtT{c}\left[\sup_{n_0 \leq k \leq n - 1}\rz_{k + 1}\right].
\]
A similar bound holds for $\ell = n + 1.$ Since $e^{-\lm(\tI{n + 1} - \tI{n_0})} \leq e^{-\lm(\tI{n} - \tI{n_0})},$ the second claim of the lemma follows.

The third claim of the lemma, bounding $\norm{\EzM(\sI{n + 1})}$, follows in a similar way to the third claim of Lemma~\ref{lem:EzDBd}.
\end{proof}

Similarly to \gal{Lemma~\ref{lem:rzBd}}, the next lemma bounds $\rt_{n + 1}$ and $\rtS_{n + 1}$ with decaying terms using Lemma~\ref{lem:EtDBd}.

\begin{lemma}[ODE-SA Distance Bound for $\theta_n$]
\label{lem:rtBd}
Fix $n_0 \geq 0$ and $n \geq n_0.$ Then on $ G_n,$ \gal{assuming \eqref{eq: Gn' definition}},
\begin{align*}
\rt_{n + 1} \leq & \Kt \norm{\EtM(\tI{n})} + \Lt{a} \; e^{-q(\tI{n} - \tI{n_0})} + \Lt{b} \left[\sup_{k \geq n_0} \beta_k\right] + \Lt{c}\left[\sup_{n_0 \leq k \leq n} \nu_{k + 1}\right], \\
\rtS_{n + 1} \leq &\Kt \norm{\EtM(\tI{n})} + [\Kt \Rti + \Lt{a}] e^{- \lm (\tI{n} - \tI{n_0})} + \Lt{b} \left[\sup_{k \geq n_0} \beta_k\right] + \Lt{c}\left[\sup_{n_0 \leq k \leq n} \nu_{k + 1}\right],
\end{align*}
where $\Lt{a} = \LtT{a}, \Lt{c} = \LtT{c}$ and $\Lt{b} := \LtD + \LtM +  \norm{\Xt} \Rti + \LtT{b}$.

\end{lemma}

\begin{proof}
Let $\tau \in [\tI{n}, \tI{n + 1}].$ Then arguing as in \gal{the} proof of Lemma~\ref{lem:rzBd}
%{ TODO: (either repeat the argument here, or make a formal claim from the argument, and refer to it both here and in Lemma~\ref{lem:rzBd})}
 and using \eqref{eqn:tLimODE}, there exists $\kappa \in [0,1]$ such that
\begin{multline*}
\norm{\bart(\tau) - \tSol{\tau}} \leq (1 - \kappa) \norm{\bart(\tI{n}) - \tSol{\tI{n}}}\\
+ \kappa \norm{\bart(\tI{n + 1}) - \tSol{\tI{n + 1}}} +  \int_{\tI{n}}^{\tI{n + 1}} \norm{\Xt} \norm{\tSol{\tau'} - \thS} \df \tau'.
\end{multline*}
\gal{Due to \eqref{eq: Gn' definition},} $\norm{\bart(\tI{n_0}) - \thS} \leq \Rti;$ thus, from Remark~\ref{rem:MonDec}, $\norm{\tSol{\tau} - \thS} \leq \Rti$ for all $t \geq \tI{n_0}.$ Using this with \eqref{eqn:TrajComp-t} and \eqref{eqn:zMatrixBd}, the facts that $\Kt \geq 1$,
\[
\alpha_n \leq \left[\sup_{k \geq n_0} \alpha_k\right] \leq \left[\sup_{k \geq n_0} \beta_k\right],
\]
and Lemma~\ref{lem:EtDBd}, the first claim of the lemma follows:
\begin{align}
\rt_{n + 1}
\leq &
\LtD\left[\sup_{k \geq n_0} \beta_{k}\right] + \LtT{a}e^{-q(\tI{n} - \tI{n_0})}
+ \LtT{b} \left[\sup_{k \geq n_0} \beta_k\right] + \LtT{c}\left[\sup_{n_0 \leq k \leq n} \nu_{k + 1}\right] \notag \\
& + \kappa \LtM \left[\sup_{k \geq n_0} \beta_k\right] + (\kappa + (1-\kappa)\Kt)\norm{\EtM(\tI{n})} + \norm{\Xt}\Rti\left[\sup_{k \geq n_0} \beta_k\right]
\notag
\\
\leq & \Kt \norm{\EtM(\tI{n})} + \Lt{a} \; e^{-q(\tI{n} - \tI{n_0})} + \Lt{b} \left[\sup_{k \geq n_0} \beta_k\right] + \Lt{c}\left[\sup_{n_0 \leq k \leq n} \nu_{k + 1}\right].
\label{eq: rho bound}
\end{align}
%TODO: improve and explain each of the relations

For the second claim of the lemma, notice that
\[
\norm{\bart(\tau) - \thS}  \leq  \norm{\bart(\tau) - \tSol{\tau}}
+ \norm{\tSol{\tau} - \thS}.
\]
Thus, we have
\[
\rtS_{n + 1} \leq \rt_{n + 1} + \sup_{\tau \in [\tI{n}, \tI{n + 1}]} \norm{\tSol{\tau} - \thS}.
\]
Lastly, \gal{using \eqref{eq: Gn' definition},} $\norm{\bart(\tI{n_0}) - \thS} \leq \Rti;$ thus, from \eqref{eqn:tOdeSol},
\[
\norm{\tSol{\tau} - \thS} \leq \Kt \Rti e^{-\lt(\tau - \tI{n_0})}.
\]
Combining the above two relations, using \eqref{eq: rho bound} and the fact that $\lm < \lt,$ the second claim of the lemma follows.
\end{proof}

\subsection{\gal{Completing the Proof of Theorem~\ref{thm:condMain}}} \label{sec: conc_bound_appendx}
We first prove Lemmas \ref{lem:rzEqEv and rtEqEv} and \ref{lem:rzSEqEv and rtSEqEv} for
{
bounding the terms appearing in Lemma~\ref{lem:IntEv} using the results from the previous subsections.
Then, we provide a bound on the martingale difference noise in Lemma~\ref{lem:MtConc and MzConc}.
Finally, we combine these results to prove Theorem~\ref{thm:condMain}.
}
%proving Lemma~\ref{lem:EvSplitMartConc}.

\begin{lemma}
\label{lem:rzEqEv and rtEqEv}
{
In accordance with Table~\ref{tab: epsilon dependent constants}, let $\N{a} \equiv \N{a}(\et,\ez,\{\alpha_k\},\{\beta_k\})$ denote the smallest positive value satisfying
\begin{equation}
\label{eq: Na def}
\max \left\{\sup_{k \geq \N{a}} \beta_k, \; \sup_{k \geq \N{a}} \eta_k \right\} \leq \frac{\min\left\{ {\et}/{8},\; {\ez}/{3}\right\}}{\Lz \max\{\Lt{c}, 1\}},
\end{equation}
$\N{b} \equiv \N{b}(\et,\{\beta_k\})$ the smallest positive value satisfying
\begin{equation}
\label{eq: Nb def}
\sup_{k \geq \N{b}} \beta_k \leq \frac{\et}{4\Lt{b}},
\end{equation}
and $N_0 \equiv N_0(\et,\ez,\{\alpha_k\},\{\beta_k\}) = \max\{\N{a},\N{b}\}$.
Then, for any $n_0 \geq N_0$ and $n \geq n_0,$}
%Fix $n_0 \geq N_0${, where $N_0$ is defined as in \eqref{defn:N0}}.
\gal{assuming \eqref{eq: Gn' definition},}%
\begin{equation}
\label{eq:rzEqEv}
[ G_n \cap \{\rz_{n + 1} \geq \ezg\}] \subseteq \left[ G_n \cap \left\{\Kz \norm{\EzM(\sI{n})} \geq \frac{\ez}{3}\right\}\right]
\end{equation}
{
and
}
\begin{multline}
\label{eq:rtEqEv}
[ G_n \cap \{\rt_{n + 1} \geq \etg\}] \\
\subseteq \left[ G_n \cap \left\{ \Kt \norm{\EtM(\tI{n})} \geq \frac{\et}{4}\right\}\right] \cup \bigcup_{k = n_0}^{n} \left[ G_k \cap \left\{ \Lt{c} \Kz \norm{\EzM(\sI{k})} \geq \frac{\et}{8}\right\}\right].
\end{multline}

\end{lemma}
\begin{proof}
Equation~\eqref{eq:rzEqEv} follows from Lemma~\ref{lem:rzBd}, \eqref{eq:epsilon smaller than R},
{ and the fact that $$2\ez/3 \geq \ez/3 \geq \Lz \max \left\{\sup_{k \geq n_0} \beta_k, \sup_{k \geq n_0} \eta_k \right\}$$ for $n_0 \geq \N{a}$.}
%for $n_0 \geq N_0$ due to \eqref{defn:N0} and \eqref{eq: Na def}.
%and { from the definition of N_0 via} \eqref{eq: Na def}.
%\sout{  and the fact that{, due to our choice of $\ez$,} $\ez/3 \leq \ez/2 \leq \ezg/2.$}

We now prove \eqref{eq:rtEqEv}.
{
Due to \eqref{eq: defn of Rto} and \eqref{eq:epsilon smaller than R}, $\etg = 4 \Lt{a}$ (see Table~\ref{tab: epsilon dependent constants} for the definition of $\Lt{a}$), and thus $\Lt{a} e^{-\lm (\tI{n} - \tI{n_0})} \leq \etg/4$ for $n \geq n_0$.
Additionally, as $n_0 \geq \N{b}$, $\Lt{b} \left[\sup_{k \geq n_0} \beta_k \right] \leq \et/4$. %$\leq \etg/4$ due to \eqref{eq:epsilon smaller than R}.
Consequently, by Lemma~\ref{lem:rtBd}, and as $\etg \geq \et$ due to \eqref{eq:epsilon smaller than R},
}
%By Lemma~\ref{lem:rtBd} and since $\etg = 4 \Lt{a},$  $\Lt{a} e^{-\lm (\tI{n} - \tI{n_0})} \leq \etg/4.$ Combined with \gal{\eqref{eq: na def}, \eqref{eq: Nb def} and \eqref{eq: Gn' definition}}, we get that
%
\begin{multline*}
[ G_n \cap \{\rt_{n + 1} \geq \etg\}] \\
\subseteq \left[ G_n \cap \left\{ \Kt \norm{\EtM(\tI{n})} \geq \frac{{\et}}{4}\right\}\right] \cup \left[ G_n \cap \left\{ \Lt{c} \left[ \sup_{n_0 \leq k \leq n} \rz_{k + 1}\right]  \geq \frac{{\et}}{4}\right\}\right].
\end{multline*}
{Noting also that %$\etg \geq \et$ due to \eqref{eq:epsilon smaller than R}, \eqref{eq: Na def} and
$G_n \subseteq G_k$ for all $n_0 \leq k \leq n$, the desired result now follows from Lemma~\ref{lem:rzBd},
%\eqref{eq:epsilon smaller than R},
and the fact that $\et/8 \geq \Lz \max \left\{\sup_{k \geq n_0} \beta_k, \sup_{k \geq n_0} \eta_k \right\}$ for $n_0 \geq \N{a}$ by the definition of $\N{a}$.}
\end{proof}

\begin{lemma}
\label{lem:rzSEqEv and rtSEqEv}
{
Fix some $n_0 \geq N_0$ and $n_1 \geq N_1,$
where, in accordance with Table~\ref{tab: epsilon dependent constants}, $N_0 \equiv N_0(\et,\ez,\{\alpha_k\},\{\beta_k\})$ is defined as in Lemma~\ref{lem:rzEqEv and rtEqEv},
$N_1 \equiv N_1(n_0,\et,\ez,\{\alpha_k\},\{\beta_k\}) = \max\{\n{a},\n{b}\}$,
$\n{a} \equiv \n{a}(n_0,\et, \{\alpha_k\})$
denotes the smallest positive value satisfying
\begin{equation}
\label{eq: na def}
[\Kt \Rti + \Lt{a}] e^{-\lm(\tI{\n{a}} - \tI{n_0})} \leq \frac{\et}{4},
\end{equation}
and
$\n{b} \equiv \n{b}(n_0,\ez, \{\beta_k\})$
denotes the smallest positive value satisfying
\begin{equation}
\label{eq: nb def}
\Kz \Rzi e^{-\lz (\sI{\n{b}} - \sI{n_0})} \leq \frac{\ez}{3}.
\end{equation}
Then\gal{, assuming \eqref{eq: Gn' definition},} for all $n \geq n_1,$
}
\begin{equation}
\label{eq:rzSEqEv}
[ G_n \cap \{\rzS_{n + 1} \geq \ez\}]
\subseteq \left[ G_n \cap \left\{\Kz \norm{\EzM(\sI{n})} \geq \frac{\ez}{3}\right\}\right]
\end{equation}
and
\begin{multline}
\label{eq:rtSEqEv}
[ G_n \cap \{\rtS_{n + 1} \geq \et\}] \\
\subseteq \left[ G_n \cap \left\{\Kt \norm{\EtM(\tI{n})} \geq \frac{\et}{4}\right\}\right] \cup \bigcup_{k = n_0}^{n} \left[ G_k \cap \left\{ \Lt{c} \Kz \norm{\EzM(\sI{k})} \geq \frac{\et}{8}\right\}\right].
\end{multline}
\end{lemma}
\begin{proof}
{
	Note  that
	\(
	\Kz \Rzi e^{-\lz(\sI{n} - \sI{n_0})} \leq \ez/3
	\)
	for all $n\geq \n{b}$ due to $\lm \leq \lz$,
	and that
	$$
	\Lz \max \left\{\sup_{k \geq n} \beta_k, \sup_{k \geq n} \eta_k \right\}  \leq \et/3
	$$
	for all $n\geq \N{a}$ (recall $\N{a}$ from Lemma~\ref{lem:rzEqEv and rtEqEv})
	. Therefore, due to Lemma~\ref{lem:rzBd},  \eqref{eq:rzSEqEv} holds.
}

{
For proving \eqref{eq:rtSEqEv}, note first that,
as $n \geq \n{a}$ and $\lm \leq \lz$, it holds that $$[\Kt \Rti + \Lt{a}] e^{-\lz(\tI{\n{a}} - \tI{n_0})} \leq \frac{\et}{4}.$$
Additionally, as $n \geq \N{b}$, $\Lt{b} \left[\sup_{k \geq n_0} \beta_k \right] \leq \et/4$
(recall $\N{b}$ from Lemma~\ref{lem:rzEqEv and rtEqEv}). %$\leq \etg/4$ due to \eqref{eq:epsilon smaller than R}.
Consequently, by Lemma~\ref{lem:rtBd},
\begin{multline}
\label{eq:rtSEqEv halfway}
[ G_n \cap \{\rtS_{n + 1} \geq \et\}] \\
\subseteq \left[ G_n \cap \left\{ \Kt \norm{\EtM(\tI{n})} \geq \frac{{\et}}{4}\right\}\right] \cup \left[ G_n \cap \left\{ \Lt{c} \left[ \sup_{n_0 \leq k \leq n} \rz_{k + 1}\right]  \geq \frac{{\et}}{4}\right\}\right].
\end{multline}
To complete the proof, we argue as in the last part of the proof of Lemma~\ref{lem:rzEqEv and rtEqEv}: noting that %$\etg \geq \et$ due to \eqref{eq:epsilon smaller than R}, \eqref{eq: Na def} and
$G_n \subseteq G_k$ for all $n_0 \leq k \leq n$, the desired result follows from \eqref{eq:rtSEqEv halfway}
using Lemma~\ref{lem:rzBd},
%\eqref{eq:epsilon smaller than R},
and the fact that $\et/8 \geq \Lz \max \left\{\sup_{k \geq n_0} \beta_k, \sup_{k \geq n_0} \eta_k \right\}$ for $n_0 \geq \N{a}$ (recall, again, $\N{a}$ from Lemma~\ref{lem:rzEqEv and rtEqEv}).
}
%Arguing as in the proof of Lemma~\ref{lem:rtEqEv}, the desired result follows from the second claim in Lemma~\ref{lem:rtBd}, first claim in Lemma~\ref{lem:rzBd}, \eqref{eq:epsilon smaller than R}, \eqref{eq: Na def}, \eqref{eq: na def}\gal{, \eqref{eq: Nb def} and \eqref{eq: Gn' definition}}.
\end{proof}

%\begin{lemma}[Bound Form for Event of Interest]
%\label{lem:EvSplitMartConc}
%Let $n_0 \geq N_0,$ $n_1 \geq N_1(n_0).$ Then\gal{, assuming \eqref{eq: Gn' definition} holds,}
%\begin{align*}
%\cE^c(&n_0,T) \subseteq  \cE_{\midE} \cup \cE_{\aftE}  \subseteq \\
%&\left[\bigcup_{n = n_0}^{\infty} \left[G_n \cap \left\{ \Kt \norm{\EtM(\tI{n})} \geq \frac{\et}{4}\right\}\right]\right]\\
%\cup &\left[\bigcup_{n = n_0}^{\infty} \left[G_n \cap \left\{ \Kz \norm{\EzM(\sI{n})} \geq \frac{\ez}{3}\right\}\right]\right]\\
%\cup &\left[\bigcup_{n = n_0}^{\infty} \left[G_n \cap \left\{ \Lt{c} \Kz \norm{\EzM(\sI{n})} \geq \frac{\et}{8}\right\}\right]\right].
%\end{align*}
%\end{lemma}
%
%\begin{proof}
%This  result follows from Lemma~\ref{lem:IntEv} and the Lemmas~\ref{lem:rzEqEv}, \ref{lem:rtEqEv}, \ref{lem:rzSEqEv}, and \ref{lem:rtSEqEv} put together.
%\end{proof}

%Lastly, to provide the proof of our main technical theorem, we give the two following lemmas.
Lastly, to provide the proof of our main technical theorem, we give the following lemma.
{
We remind the reader that $a_n = \sum_{k = 0}^{n - 1} \st_k^{2} e^{ -2 \lt (\tI{n} - \tI{k + 1})},$
and $b_n := \sum_{k = 0}^{n - 1} \sw_k^{2} e^{ -2 \lz (\sI{n} - \sI{k + 1})}$ for $n \geq 0$.
Also recall that $\EtM(\tI{n})$ and $\EzM(\tI{n})$ depend on $n_0$, as can be seen from their definition in \gal{Subsection~\ref{subsec:Comparison}}.
}

\begin{lemma}[Azuma-Hoeffding for $\EtM$ { and $\EzM$}]
\label{lem:MtConc and MzConc}
Fix $n_0 \geq 0,$ $\delta > 0.$ Then for any $n \geq n_0,$
\begin{equation}
\label{eq:MtConc}
\Pr\left\{G_n,  \norm{\EtM(\tI{n})} \geq \delta\right\} \leq 2d^2 \exp\left(-\frac{\delta^2}{d^3 (\LtM)^2 a_n}\right)
\end{equation}
{
and
\begin{equation}
\label{eq:MzConc}
\Pr\left\{G_n,  \norm{\EzM(\sI{n})} \geq \delta\right\} \leq 2d^2 \exp\left(-\frac{\delta^2}{d^3 (\LzM)^2 b_n}\right).
\end{equation}
}
\end{lemma}
\begin{proof}
{
We only prove \eqref{eq:MtConc}; \eqref{eq:MzConc} follows similarly.
}

Let $A_{k, n}$ be the matrix $\int_{t_k}^{t_{k + 1}} e^{- \Xt(\tI{n} - \tau)}\df \tau$ with $A_{k, n}^{ij}$ denoting its $i,j-$th entry. Let $\Mt_{k + 1} (j)$ denote the $j-$th entry of $\Mt_{k + 1}.$ On $G_{n},$ $1_{G_k} = 1$ for all $n_0 \leq k \leq n.$ So
\begin{eqnarray*}
\Pr\left\{G_n,  \norm{\EtM(\tI{n})} \geq \delta\right\} & = & \Pr\left\{G_n, \norm{\sum_{k = n_0}^{n - 1}A_{k, n} \Mt_{k + 1} 1_{G_k}} \geq \delta\right\}\\
& \leq & \Pr\left\{\norm{\sum_{k = n_0}^{n - 1}A_{k, n} \Mt_{k + 1} 1_{G_k}} \geq \delta\right\}\\
& \leq &  \sum_{i = 1}^{d}\sum_{j = 1}^{d} \Pr\left\{\norm{\sum_{k = n_0}^{n - 1}A^{ij}_{k, n} \Mt_{k + 1}(j) 1_{G_k}} \geq \frac{\delta}{d \sqrt{d}}\right\},
\end{eqnarray*}
where the last relation is due to the union bound applied twice. On $G_k,$ $\Kt \norm{\Mt_{k + 1}} \leq \LtM.$ Hence, on $G_k$, for any $i, j \in \{1, \ldots, d\},$ using  \eqref{eqn:tMatrixBd},
\[
|A_{k, n}^{ij}| \; |\Mt_{k + 1}(j)| \leq \norm{A_{k, n}} \; \norm{M_{k + 1}} \leq \Kt \LtM \st_k e^{-\lt(\tI{n} - \tI{k + 1})}.
\]
Using $\sum_{k = n_0}^{n - 1} \st_k^2 e^{-2\lt(\tI{n} - \tI{k + 1})} \leq a_n,$ the desired result now follows from the Azuma-Hoeffding inequality.
\end{proof}

%\begin{lemma}[Azuma-Hoeffding for $\EzM$]
%\label{lem:MzConc}
%Fix $n_0 \geq 0,$ $\delta > 0.$ Then for any $n \geq n_0,$
%
%\[
%\Pr\left\{G_n,  \norm{\EzM(\sI{n})} \geq \delta\right\} \leq 2d^2 \exp\left(-\frac{\delta^2}{d^3 [\LzM]^2 b_n}\right).
%\]
%
%\end{lemma}
%\begin{proof}
%The proof follows similarly to that of Lemma~\ref{lem:MtConc}.
%\end{proof}

We finish with combining the above lemmas for proving our main technical result.

\begin{proof}\textbf{of Theorem~\ref{thm:condMain}}
{
Lemmas~\ref{lem:IntEv},~\ref{lem:rzEqEv and rtEqEv}, and \ref{lem:rzSEqEv and rtSEqEv} together show that, for \gal{any $n_0 \geq N_0(\et,\ez,\{\alpha_k\},\{\beta_k\})$ and $n_1 \geq N_1(n_0,\et,\ez,\{\alpha_k\},\{\beta_k\}),$}
\begin{multline*}
\cE^c(n_0,T) \subseteq
%{\cE_{\midE} \cup \cE_{\aftE}  \subseteq }}
\left[\bigcup_{n = n_0}^{\infty} \left[G_n \cap \left\{ \Kt \norm{\EtM(\tI{n})} \geq \frac{\et}{4}\right\}\right]\right]\\
\cup \left[\bigcup_{n = n_0}^{\infty} \left[G_n \cap \left\{ \Kz \norm{\EzM(\sI{n})} \geq \frac{\ez}{3}\right\}\right]\right] \cup \left[\bigcup_{n = n_0}^{\infty} \left[G_n \cap \left\{ \Lt{c} \Kz \norm{\EzM(\sI{n})} \geq \frac{\et}{8}\right\}\right]\right].
\end{multline*}
The proof then follows from
%Lemmas~\ref{lem:EvSplitMartConc}
Lemma~\ref{lem:MtConc and MzConc}.
}
%\ref{lem:MtConc} and \ref{lem:MzConc}.
\end{proof}

%\subsection{Proofs from Subsection~\ref{subsection: analysis preliminaries}}

\section{\gal{Proof of Theorem~\ref{thm:SparseProj}}}
\label{sec:Proof_Sparse_Proj}
\gal{Using Theorem~\ref{thm:condMain}, we are now ready to prove Theorem~\ref{thm:SparseProj}.}

\begin{proof}\textbf{of Theorem~\ref{thm:SparseProj}, Statement~\ref{st:ProbEst}}
{
First we claim that, under the choice of stepsize in the statement of the theorem, we have
\begin{equation}
\label{eq: n0' satisfies N0 constraint}
n_0' \geq N_0(\epsilon,\epsilon,\{\alpha_k\},\{\beta_k\}) \enspace,
\end{equation}
\gugan{where $N_0(\epsilon, \epsilon, \{\alpha_k\}, \{\beta_k\})$ is as in Lemma~\ref{lem:rzEqEv and rtEqEv}.
The reason for this is that, due to our choice of $n_0'$,
\eqref{eq: Na def} and \eqref{eq: Nb def} hold with
\gugan{
\begin{align*}
\N{a}(\epsilon,\epsilon,\{\alpha_k\},\{\beta_k\}) = &
\left[8\Lz \max\{\Lt{c}, 1 \} / \epsilon \right]^{\frac{1}{\min\{\beta,\alpha-\beta\}}},
\\
\N{b}(\epsilon,\epsilon,\{\beta_k\}) = &
\left[4\Lt{b}/\epsilon\right]^{1/\beta}.
\end{align*}
}
}
}

Additionally, \gugan{for any $n_0,$} \eqref{eq: nb def} holds with
\begin{align*}
\n{b}(n_0,\epsilon, \{\beta_k\})
=&
\left[(n_0+1)^{1-\beta} + \tfrac{1-\beta}{\lz}\ln\left[\tfrac{3\Kz \Rzi}{\epsilon}\right]\right]^{\frac{1}{1-\beta}} \enspace.
\end{align*}
\gugan{
This follows from the fact that
\begin{eqnarray}
\sum_{k=n_0}^{\n{b} -1} (1+k)^{-\beta} & \geq & \int_{n_0}^{\n{b}} (1+x)^{-\beta} dx \\
& = & \frac{1}{1-\beta}\left[ (\n{b}+1)^{(1-\beta)} - (\n{0}+1)^{(1-\beta)}\right] \enspace.
\end{eqnarray}}Similarly, \eqref{eq: na def} holds with
\[
\n{a}(n_0,\epsilon, \{\alpha_k\}) =
\left[(n_0+1)^{1-\alpha} +\tfrac{1-\alpha}{\lm}\ln \left[\tfrac{4[\Kt \Rti + \Lt{a}]}{\epsilon}\right]\right]^{\frac{1}{1-\alpha}} \; .
\]

\gugan{
For all $n_0 \geq 3,$ we have  $2n_0 \geq 1.5(n_0 + 1).$ Hence, if
\[
n_0 \geq \max\left\{\left[\frac{1-\alpha}{((1.5)^{1-\alpha}-1)\lm}\ln\frac{4[\Kt \Rti + \Lt{a}]}{\epsilon}\right]^{1/(1-\alpha)},\; 3\right\},
\]
then it is easy to see that $2 n_0 \geq \n{a}(n_0, \epsilon, \{\beta_k\})$. Similarly, if
\[
n_0 \geq \max\left\{\left[\frac{1-\beta}{((1.5)^{1-\beta}-1)\lz}\ln\frac{3\Kz \Rzi}{\epsilon}\right]^{1/(1-\beta)},\; 3\right\},
\]
then $2 n_0  \geq \n{a}(n_0, \epsilon, \{\beta_k\}).$ Thus, by our} choice of $n_0',$
\begin{equation}
\label{eq: n0' satisfies N1 constraint}
2n_0' \geq N_1(n_0',\epsilon,\epsilon,\{\alpha_k\},\{\beta_k\}) \enspace,
\end{equation}
\gugan{where $N_1(n_0', \epsilon, \epsilon, \{\alpha_k\}, \{\beta_k\})$ is as in Lemma~\ref{lem:rzSEqEv and rtSEqEv}.}

By \eqref{eq: w assumption}, we have
\gugan{
\[
\{w \in \dReal:\|w\| \leq \Rzi/2\} \subseteq \{w \in \dReal: \|w-\lambda(\theta)\| \leq \Rzi \; \; \forall \theta \mbox{ with } \|\theta\| \leq \Rti / 2 \} \enspace.
\] Combining this with using \eqref{eq: theta assumption}, \eqref{eq: z_n def}, and since} ${n_0'}$ is a power of $2,$ it follows from the definition of the projection operation that
\begin{equation}
\label{eqn:InitCond}
\|\theta_{n_0'}' - \thS\| \leq \Rti, \text{ and } \|z_{n_0'}'\| \leq \Rzi \enspace.
\end{equation}

Let $(\theta_n, w_n)_{n \geq n_0'}$ be the iterates obtained by running the unprojected algorithm given in \eqref{eq:theta_iter} and \eqref{eq:w_iter} with $\theta_{n_0'} = \theta_{n_0'}'$ and $w_{n_0'} = w_{n_0'}'.$ Because of \eqref{eqn:InitCond}, \gugan{it follows that \eqref{eq: Gn' definition} holds. Combining this with \eqref{eq: n0' satisfies N0 constraint} and \eqref{eq: n0' satisfies N1 constraint}, it follows from Theorem~\ref{thm:condMain} that}
\gugan{
\begin{align}
&\Pr\{
\norm{\theta_n - \thS} \leq \epsilon, \norm{z_n} \leq \epsilon, \forall\; n \geq 2n_0'\}
\notag\\
&\geq
1 -2d^2 \! \!
\sum_{n \geq n_0'}\left[
\exp\!\left[\tfrac{-c_1 \et^2}{a_n}\right] \!
+
\exp\!\left[\tfrac{-c_2 \et^2}{b_n}\right]\!
+
\exp\!\left[\tfrac{-c_3\ez^2}{b_n}\right]
\right]
\notag\\
&
\geq
1 -2d^2 \! \!
\sum_{n \geq n_0'}\left[
\exp\!\left[\tfrac{-c_1 \et^2}{a_n}\right] \!
+
2\exp\!\left[\tfrac{-\min(c_2, c_3) \et^2}{b_n}\right]\!
\right]\!.	
\label{eq: concentration of original iterates}
\end{align}
}

As the next step, we claim that, \gugan{for any $n,$ the event}
\begin{equation}
\label{eq: projection is identity}
\{\norm{\theta_{n} - \thS} \leq \epsilon, \norm{z_{n}} \leq \epsilon\}
\subseteq
\{\theta_{n} = \Pi_{n,\Rti/2}(\theta_{n}), w_{n}=\Pi_{n,\Rzi/2}(w_{n})\} \enspace.
\end{equation}
Indeed, \gugan{due to} \eqref{eq: theta assumption} and the choice of $\epsilon$, $\norm{\theta_{n}-\thS} \leq \epsilon$ implies
\begin{equation}
\label{eq: theta bounded in Rti}
\norm{\theta_{n}} \leq \norm{\theta_{n} - \thS} + \norm{\thS} \leq \epsilon + \Rti/4 \leq \Rti/2
\end{equation}
and thus $\theta_{n} = \Pi_{n,\Rti/2}(\theta_{n})$. Separately, from the above relation and \eqref{eq: w assumption}, we also have $\|\lambda(\theta_{n})\|\leq\Rzi/4.$ \gugan{Because of this,} \eqref{eq: z_n def}, the fact that $\norm{z_{n}} \leq \epsilon,$ and the choice of $\epsilon,$ it then follows that  $\norm{w_{n}} \leq \norm{\lambda(\theta_{n})}+\norm{z_n} \leq \Rzi/2$, and thus $w_{n} = \Pi_{n, \Rzi/2}(w_{n})$.

{
An immediate consequence of \eqref{eq: projection is identity} is that the event
\begin{eqnarray}
\mathcal{I} & := & \{\norm{\theta_j - \thS} \leq \epsilon, \norm{z_j} \leq \epsilon, \forall\; j \geq \textcolor{blue}{2n_0'}\} \notag
\\
& \subseteq &
\{\theta_{j} = \Pi_{j,\Rti}(\theta_{j}), w_{j}=\Pi_{j,\Rzi}(w_{j}), \forall\; j \geq \textcolor{blue}{2n_0'}\} \enspace. \label{eq: projection is identity 2}
\end{eqnarray}
The statement of the theorem now follows by an easy coupling argument.
For this, let
\begin{align}
(\tilde{\theta}_{n}',\tilde{w}_{n}')
:=
\begin{cases}
({\theta}_n',{w}_n'), & \mbox{ for } 0 \leq n < n_0' \enspace,\\
({\theta}_n,{w}_n), & \mbox { for } n \geq n_0' \mbox{ on the event $\mathcal{I}$ },\\
({\theta}_n',{w}_n'), & \mbox{ for } n \geq  n_0' \mbox{ on the complement of the event $\mathcal{I}\enspace.$ }
\end{cases}
\end{align}
Due to \eqref{eq: projection is identity 2}, $(\tilde{\theta}_n',\tilde{w}_n')_{n\geq 0}$ and $(\theta'_n, w'_n)_{n\geq 0}$ are distributed identically.
%, conditioned on the event $\left(\{\theta_{n_0'}=\theta'_{n_0'},\; w'_{n_0'} = w_{n_0'}\}\cap\bigcap_{n\geq 2n_0'-1}I_n\right)$.
This, together with \eqref{eq: concentration of original iterates} and Lemma~\ref{lem: B formula}, completes the proof of the claimed result.
}
\end{proof}

\begin{proof}\textbf{of Theorem~\ref{thm:SparseProj}, Statement~\ref{st:ConvRate}}
\gugan{Let
\[
N_0''(\epsilon,\delta,\st,\sw) =
\max \left\{ \left[\frac{1}{c_{6a} \epsilon^2} \log \frac{4d^2 c_{7a} e^{c_{5a}\epsilon^2}}{\epsilon^{2/\alpha} \delta}
\right]^{1/\st},
\left[\frac{1}{c_{6b} \epsilon^2} \log \frac{8d^2 c_{7b} e^{c_{5b}\epsilon^2}}{\epsilon^{2/\beta} \delta}
\right]^{1/\sw} \right\}  \enspace.
\]
}
Obviously, for any $n_0' \geq N_0''(\epsilon,\delta,\st,\sw)$,
\[
2d^2 \frac{c_{7a}}{\epsilon^{2/\alpha}} \exp\left[c_{5a} \epsilon^2 - c_{6a} \epsilon^2 (n_0')^\alpha \right] + 4d^2 \frac{c_{7b}}{\epsilon^{2/\beta}} \exp\left[c_{5b} \epsilon^2 - c_{6b} \; \epsilon^2 (n_0')^\beta \right] \leq \delta \enspace.
\]
Therefore, by Theorem~\ref{thm:SparseProj}, Statement~\ref{st:ProbEst},
\begin{equation}
\label{eq: comparison to delta}
\Pr\{\norm{\theta_n' - \thS} \leq \epsilon, \norm{z_n'} \leq \epsilon, \forall n \geq 2n_0'\} \geq 1-\delta
\end{equation}
for any $n_0' \geq \max\{N_0'(\epsilon,\st,\sw),N_0''(\epsilon,\delta,\st,\sw)\}$ \gal{such that $n_0'$ is a power of 2. Thus,}
\begin{equation}
\Pr\{\norm{\theta_n' - \thS} \leq \epsilon, \norm{z_n'} \leq \epsilon, \forall n \geq n_0'\}
\geq
1-\delta
\label{eq: pre-rate bound}
\end{equation}
for any $n_0' \geq 4\max\{N_0'(\epsilon,\st,\sw),N_0''(\epsilon,\delta,\st,\sw)\}$. \gal{The factor $4$ appears because the $2n_0'$ in \eqref{eq: comparison to delta} is replaced with $n_0'$ in \eqref{eq: pre-rate bound}, and the fact that $n_0'$ was earlier required to be a power of 2.}

%Due to the construction of functions $N_0'(\epsilon,\st,\sw)$ and $N_0''(\epsilon,\delta,\st\sw)$, it holds that, for any
{For any integer $n > 3,$ we argue that there is some $\epsilon \equiv \epsilon(n)$
such that
\gugan{
\[
n = 4\max\{N_0'(\epsilon,\st,\sw),N_0''(\epsilon,\delta,\st,\sw)\} \enspace;
\]}indeed, as $N_0'(\epsilon,\st,\sw)$ and $N_0''(\epsilon,\delta,\st, \sw)$ are both defined to be the maximum of terms that strictly monotonically inrease as $\epsilon$ decreases---except for the constant $3$ in \eqref{eq: N0' defn}---such an $\epsilon(n)$ exists. Furthermore, it is also not difficult easy to see} that
\begin{equation}
%  f(x) = O(x^{-1/\min(\sw,\st-\sw)}) \enspace.
{ \epsilon(n) = O\left(\max\left\{n^{-\beta/2}\sqrt{{\ln (n/\delta)}},\; n^{\sw-\st)}\right\}\right) \enspace.}
\label{eq: order of magnitude of f}
\end{equation}
This, together with \eqref{eq: pre-rate bound}, implies
\begin{equation}
\Pr\{\norm{\theta_n' - \thS} \leq {\epsilon}(n), \norm{z_n'} \leq {\epsilon}(n)\}
\geq
1-\delta
\end{equation}
for any $n > 3,$ completing the proof.
\end{proof}

\section{Proofs from Section~\ref{sec:Appl}} \label{sec: RL appendix}
\gal{Similarly to GTD(0) in Section~\ref{sec:Appl}, we now show how our assumptions hold, and with what constants, for GTD2 and TDC algorithms. Thus, in the same spirit as Corollary~\ref{cor: RL}, similar results trivially follow for these algorithms as well.}
\subsection{GTD2}
The GTD2 \; algorithm~\citep{sutton2009fast} minimizes the objective function
\begin{align}
J^{\rm MSPBE}(\theta)
%&=
%\tfrac{1}{2}\|\bE[V_\theta-\Pi T^{\pi}V_\theta]\|_2^2
%\notag\\
&=
\tfrac{1}{2}(b-A\theta)^\top C^{-1}(b-A\theta).
\label{eq: MSPBE}
\enspace
\end{align}

The update rule of the algorithm takes the form of Equations \eqref{eq:theta_iter} and \eqref{eq:w_iter} with
\begin{align*}
h_1(\theta,w)  &= %\bE\left[\left(\phi - \gamma \phi'\right)\phi^\top\right]w =
A^\top w,
\\
h_2(\theta,w)  &= %\bE\left[ \left(\delta - \phi^\top w\right) \phi \right] =
b-A\theta -Cw,
\end{align*}
and
%correspondingly,
\begin{align*}
\Mt_{n+1} =& \left(\phi_n - \gamma \phi_n'\right)\phi_n^\top w_n - A^\top w_n \enspace,
\\
\Mw_{n+1} %=& \left(\delta_n - \phi_n^\top w_n\right) \phi_n - [b-A\theta_n -Cw_n]\\
=& r_n\phi_n + \phi_n[\gamma\phi_n'-\phi_n]^\top\theta_n - \phi_n\phi_n^\top w_n - [b-A\theta_n -Cw_n]\enspace.
\end{align*}
That is, in case of GTD2\ the relevant matrices in the update rules take the form
$\Tt = 0$, $\Wt = -A^\top$, $v_1 = 0$, and $\Tw=A$, $\Ww=C$, $v_2=b$.
Additionally, $\Xt = \Tt - \Wt\Ww^{-1}\Tw = A^\top C^{-1} A$.
By our assumptions, both $\Ww$ and $\Xt$ are symmetric positive definite matrices, and thus the real part of their eigenvalues are also positive.
It is also clear that
\begin{eqnarray*}
\|\Mt_{n+1}\| & \leq& (1+\gamma+\|A\|) \|w_n\| ,\\
\|\Mw_{n+1}\| & = & \|r_n\phi_n-b+[A+\phi_n(\gamma\phi'_n-\phi_n)^\top]\theta_n - [\phi_n\phi_n^\top-C] w_n\|\\
& \leq & 1+\|b\| + (1+\gamma+\|A\|)\|\theta_n\| + (1+\|C\|)\|w_n\|.
\end{eqnarray*}
Consequently, Assumption~\ref{assum:Noise} is satisfied with constants $\mt = (1+\gamma+\|A\|)$ and $\mw = 1 + \max(\|b\|,\gamma+\|A\|,\|C\|)$.

\subsection{TDC}

The TDC algorithm is designed to minimize (\ref{eq: MSPBE}), just like GTD2.

The update rule of the algorithm takes the form of Equations \eqref{eq:theta_iter} and \eqref{eq:w_iter} with
\begin{align*}
h_1\theta(\theta,w) &= %\bE\left[\delta\phi - \gamma \phi'\phi^\top w\right] =
b-A\theta + [A^\top-C] w \enspace,
\\
h_2(\theta,w) &= %\bE\left[ \left(\delta - \phi^\top w\right) \phi \right] =
b-A\theta -Cw \enspace,
\end{align*}
and
% correspondingly,
\begin{align*}
\Mt_{n+1} %=& \delta_n\phi - \gamma \phi'\phi^\top w_n - [b-A\theta_n + [A^\top-C] w_n] \\
=& r_n\phi_n + \phi_n[\gamma\phi_n'-\phi_n]^\top\theta_n - \gamma \phi'\phi^\top w_n - [b-A\theta_n + [A^\top-C] w_n]
\enspace,
\\
\Mw_{n+1} %=& \left(\delta_n - \phi_n^\top w_n\right) \phi_n - [b-A\theta_n +Cw_n] \\
=& r_n\phi_n + \phi_n[\gamma\phi_n'-\phi_n]^\top\theta_n - \phi_n\phi_n^\top w_n - [b-A\theta_n +Cw_n]
\enspace.
\end{align*}
That is, in case of TDC, the relevant matrices in the update rules take the form
$\Tt = A$, $\Wt = [C-A^\top]$, $v_1 = b$, and $\Tw=A$, $\Ww=C$, $v_2=b$.
Additionally, $\Xt = \Tt - \Wt\Ww^{-1}\Tw = A - [C-A^\top] C^{-1} A = A^\top C^{-1} A$.
By our assumptions, both $\Ww$ and $\Xt$ are symmetric positive definite matrices, and thus the real part of their eigenvalues are also positive.
It is also clear that
\begin{align*}
\|\Mt_{n+1}\| \leq&  2+(1+\gamma+\|A\|) \|\theta_n\| +(\gamma+\|A\|+\|C\|) \|w_n\|, \\
\|\Mw_{n+1}\| =& 2+(1+\gamma+\|A\|) \|\theta_n\|+(1+\|C\|) \|w_n\| \enspace.
\end{align*}
Consequently, Assumption~\ref{assum:Noise} is satisfied with constants $\mt = (2+\gamma+\|A\|+\|C\|)$ and $\mw = (2+\gamma+\|A\|+\|C\|)$.

\end{document}